\pdfoutput=1

\documentclass[11pt]{article}

\usepackage{emnlp22}

\usepackage{times}
\usepackage{latexsym}
\usepackage{inconsolata}

\usepackage[T1]{fontenc}
\usepackage{hyperref}       
\usepackage{url}            
\usepackage{booktabs}       
\usepackage{amsfonts}       
\usepackage{nicefrac}       
\usepackage{microtype}      
\usepackage{lipsum}		
\usepackage{graphicx}
\usepackage{natbib}
\usepackage{doi}

\usepackage{microtype}
\usepackage{graphicx}
\usepackage{subfigure}
\usepackage{booktabs}
\usepackage{hyperref}
\usepackage{caption}

\usepackage{amsmath}
\usepackage{amssymb}
\usepackage{mathtools}
\usepackage{amsthm}
\usepackage{thm-restate}
\usepackage{cleveref}
\crefname{section}{\S}{\S\S}
\Crefname{section}{\S}{\S\S}
\crefformat{section}{\S#2#1#3}
\crefname{figure}{Fig.}{Fig.}
\crefname{alg}{Alg.}{Alg.}
\crefname{thm}{Theorem}{Theorems}
\crefname{line}{line}{lines}
\crefname{appendix}{App.}{}
\crefname{equation}{Eq.}{Eq.}
\crefname{defin}{Def.}{Defs.}
\crefname{tab}{Table}{Tables}
\crefname{prop}{Proposition}{Propositions}

\usepackage[disable, textsize=tiny]{todonotes}
\newcommand{\note}[4][]{\todo[author=#2,color=#3,size=\scriptsize,fancyline,caption={},#1]{#4}} 

\newcommand{\ryan}[2][]{\note[#1]{ryan}{violet!40}{#2}}

\usepackage{xcolor}

\usepackage{xcolor}
\usepackage{colortbl}
\definecolor{ETHBlue}{RGB}{33,92,175}   
\definecolor{ETHGreen}{RGB}{98,115,19}      
\definecolor{ETHPurpleDark}{RGB}{140,10,89} 
\definecolor{ETHPurple}{RGB}{163,7,116} 
\definecolor{ETHGray}{RGB}{111,111,111} 
\definecolor{ETHRed}{RGB}{183,53,45}    
\definecolor{ETHPetrol}{RGB}{0,120,148} 
\definecolor{ETHBronze}{RGB}{142,103,19}    
\definecolor{ETHOrange}{RGB}{230, 100, 50}

\colorlet{MacroColor}{black}
\colorlet{MACROCOLOR}{black}

\newcommand{\mymacro}[1]{{\color{MacroColor} #1}}

\newcommand{\vbeta}{\boldsymbol \beta}
\newcommand{\vtheta}{{ \boldsymbol \theta}}
\newcommand{\vx}{\mathbf{x}}

\newcommand{\vw}{\boldsymbol{w}}
\newcommand{\hilbert}{\mathcal{H}}
\newcommand{\R}{\mathbb{R}}
\newcommand{\xx}{\boldsymbol{x}}
\newcommand{\yy}{\boldsymbol{y}}
\newcommand{\zz}{\boldsymbol{z}}

\newcommand{\xxn}{\xx_n}
\newcommand{\xxm}{\xx_m}

\newcommand{\xxi}{\xx_i}
\newcommand{\xxj}{\xx_j}

\newcommand{\defn}[1]{\textbf{#1}}

\newcommand{\phx}{\boldsymbol{\Phi}(\xx_n)}

\newcommand{\phxnprime}{\boldsymbol{\Phi}(\xx_{n'})}

\newcommand{\PP}{\mymacro{\boldsymbol{P}}}
\newcommand{\II}{\mymacro{\boldsymbol{I}}}
\newcommand{\WW}{\mymacro{\boldsymbol{W}}}
\newcommand{\UU}{\mymacro{\boldsymbol{U}}}
\newcommand{\SSigma}{\mymacro{\boldsymbol{\Sigma}}}
\newcommand{\targetdimension}{\mymacro{K}}

\newcommand{\phz}{\boldsymbol{\Phi}(\zz)}
\newcommand{\phxn}{\boldsymbol{\Phi}(\xxn)}
\newcommand{\phxnyst}{\boldsymbol{\widetilde{\Phi}}(\xx_n)}
\newcommand{\phnyst}{\boldsymbol{\widetilde{\Phi}}}

\newcommand{\kernel}{\mymacro{\kappa}}
\newcommand{\bigO}{\mathcal{O}}
\newcommand{\K}{\mymacro{\boldsymbol{K}}}
\newcommand{\Km}{\mymacro{\boldsymbol{K}^{(m)}}}
\newcommand{\Kmn}{\mymacro{\boldsymbol{K}^{(m, n)}}}
\newcommand{\rank}{\mymacro{\mathrm{rank}}}

\newcommand{\thetaphi}{\langle \vtheta, \phzproj \rangle}

\newcommand{\phzproj}{{\boldsymbol \Phi_{\text{proj}}(\zz)}}
\newcommand{\phiproj}{{\boldsymbol \Phi_{\text{proj}}}}

\newcommand{\vphi}{\boldsymbol{\Phi}}
\newcommand{\vphitilde}{\widetilde{\boldsymbol{\Phi}}}
\newcommand{\phxi}{\boldsymbol{\Phi}(\vx_i)}

\newcommand{\degree}{\texttt{d}}
\newcommand{\Pmat}{\mathrm{P}}

\newcommand{\projcomp}{\mathrm{P}_{\vw}^{\perp}}
\newcommand{\identity}{\mathrm{I}}

\newcommand{\valpha}{\boldsymbol{\alpha}}

\newcommand{\nystrom}{Nystr\"{o}m}
\newtheorem{theorem}{Theorem}
\newtheorem{proposition}{Proposition}

\newcommand{\sumi}{\sum_{n=1}^N}

\newcommand{\flambda}{f_{\lambda}}

\newcommand{\poly}{{\small \textsf{Poly}}}
\newcommand{\rbf}{{\small \textsf{RBF}}}
\newcommand{\laplace}{{\small \textsf{Laplace}}}
\newcommand{\linear}{{\small \textsf{Linear}}}
\newcommand{\mlp}{{\small \textsf{MLP}}}
\newcommand{\sigmoid}{{\small \textsf{Sigmoid}}}
\newcommand{\multiple}{{\small \textsf{Multiple}}}
\newcommand{\easymkl}{{\small \textsf{EasyMKL}}}
\newcommand{\uniformmk}{{\small \textsf{UniformMK}}}
\newcommand{\defeq}[0]{\mathrel{\stackrel{\textnormal{\tiny def}}{=}}}

\newcommand{\spanmat}{\texttt{span}}

\DeclareMathOperator*{\argmin}{\mathop{\text{argmin}}}

\usepackage{microtype}

\newcommand{\loss}{\mymacro{\ell}}

\newcommand{\yn}{y_n}

\newcommand{\projk}{\mathcal{P}_k}
\newcommand{\ynhat}{\widehat{y}_n}
\newcommand{\dataset}{{\mathcal{D}}}

    \theoremstyle{plain}
    \newtheoremstyle{TheoremNum}
        {\topsep}{\topsep}              
        {\itshape}                      
        {}                              
        {\bfseries}                     
        {.}                             
        { }                             
        {\thmname{#1}\thmnote{ \bfseries #3}}
    \theoremstyle{TheoremNum}

\title{Kernelized Concept Erasure}
 
 \author{Shauli Ravfogel\textsuperscript{\normalfont1,2} \, Francisco Vargas\textsuperscript{\normalfont 3} \, Yoav Goldberg\textsuperscript{\normalfont1,2}  \, Ryan Cotterell\textsuperscript{\normalfont 4}\\
\textsuperscript{1}Bar-Ilan University \, \textsuperscript{2}Allen Institute for Artificial Intelligence \\
\textsuperscript{3}University of Cambridge \, \textsuperscript{4}ETH Zürich \\ 
  {\tt\{\href{mailto:shauli.ravfogel@gmail.com}{shauli.ravfogel}, \href{mailto:yoav.goldberg@gmail.com}{yoav.goldberg}\}@gmail.com} \\
   \tt{\href{mailto:fav25@cam.ac.uk}{fav25@cam.ac.uk}}  \tt{\href{mailto:ryan.cotterell@inf.ethz.ch}{ryan.cotterell@inf.ethz.ch}}
  }
\begin{document}
\maketitle
\begin{abstract}

The representation space of neural models for textual data emerges in an unsupervised manner during training.
Understanding how those representations encode human-interpretable concepts is a fundamental problem.
One prominent approach for the identification of concepts in neural representations is searching for a linear subspace whose erasure prevents the prediction of the concept from the representations.
However, while many linear erasure algorithms are tractable and interpretable, neural networks do not necessarily represent concepts in a linear manner. To identify non-linearly encoded concepts, we propose a kernelization of a linear maximin game for concept erasure. We demonstrate that it is possible to prevent specific non-linear adversaries from predicting the concept. However, the protection does not transfer to different non-linear adversaries. Therefore, exhaustively erasing a non-linearly encoded concept remains an open problem.\looseness=-1

\vspace{.5em}
\hspace{.25em}\includegraphics[width=1.25em,height=1.25em]{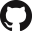}\hspace{.75em}\parbox{\dimexpr\linewidth-4\fboxsep-2\fboxrule}{\sloppy\small \url{https://github.com/shauli-ravfogel/adv-kernel-removal}
}

\end{abstract}

\newcommand{\word}[1]{\textsc{#1}}
\section{Introduction}

Large neural networks in NLP produce real-valued representations that encode the bit of human language that they were trained on, e.g., words, sentences, or text grounded in images.
For instance, GloVe \citep{DBLP:conf/emnlp/PenningtonSM14} produces real-valued representations of isolated words, BERT \citep{devlin-etal-2019-bert} produces real-valued representations of sentences, and VilBERT produces real-valued representations of visually grounded language \citep{vilbert, bugliarello-etal-2021-multimodal}.
These real-valued representations naturally encode various properties of the objects they represent.
For instance, a good representation of the first author's laptop computer ought to encode its manufacturer, size, and color somewhere among its real values.\looseness=-1

We now describe the premise of our paper in more detail.
We adopt the notion of a \defn{concept} due to \citet{gardenfors}.
For G{\"a}rdenfors, objects can be thought of as having a geometric representation.
Different dimensions in the representation space that objects inhabit might correspond to their color, size, and shape.
G{\"a}rdenfors then goes further and defines a \defn{concept} as a convex region of the representation space.\footnote{In this work, we exclusively focus on linear subspaces, which are convex, rather than more general convex regions.}
Building on G{\"a}rdenfors' notion of a concept, this paper studies a task that we refer to as \defn{concept erasure}.

 We now motivate the task more formally by extending
 the example.
 Imagine we have $N$ different laptops, whose real-valued representations are denoted as $\xx_1, \dots, \xx_N$.
 Now, consider concept labels $y_1,\ldots, y_N$ that encode each laptop's color and are taken from the set $\{\textsc{grey}, \textsc{silver}, \textsc{black}, \textsc{white}\}$.
In the concept erasure paradigm, we seek an \defn{erasure function} $r(\cdot)$ such that $r(\xx_1),\ldots,r(\xx_N)$ are \emph{no longer} predictive of the colors of the laptops, $y_1,\ldots,y_N$, but \emph{retain} all the other information encoded in the original representations, i.e., they remain predictive with respect to other laptop-related concepts.
Our hope is that the geometry of the erasure function $r(\cdot)$ then tells us the structure of the laptop color concept.\looseness=-1

Concept erasure is tightly related to concept identification.
Once we have successfully removed a given concept, e.g., color, from a representation $\xx_n$, it is reasonable to argue that the erasure function $r$ has meaningfully identified the concept within the representation space. In the rest of the paper, we will say that $r(\cdot)$ \defn{neutralizes} the concept in the representation space.
For instance, we say that the $r$ in our example neutralizes the concept of a laptop's color.
It follows that concept identification is related to bias mitigation \cite{bolukbasi2016man,gonen-goldberg-2019-lipstick,hall-maudslay-etal-2019-name}, e.g., one may want to identify and remove the gender encoded in learned word representations produced by an embedding method such as word2vec \citep{DBLP:journals/corr/abs-1301-3781} or GloVe \citep{ DBLP:conf/emnlp/PenningtonSM14}.
Indeed, the empirical portion of this paper will focus on removing gender from word representations in order to mitigate bias.\looseness=-1

Previous work on concept erasure \citep{ravfogel-etal-2021-counterfactual} focuses on the \emph{linear} case, i.e., where $r$ is a linear function. While linear concept erasure methods have certainly found success \cite{bolukbasi2016man}, there is no \textit{a-priori} reason to suspect that neural networks encode concepts in a linear manner. In this work, we take the first step toward the goal of identifying a \emph{non-linear} function $r(\cdot)$ and a corresponding non-linearly encoded concept subspace.
We directly build on \citet{rlace}, who cast linear concept erasure as a maximin game. 
Under their formulation, the function $r(\cdot)$ learns to remove the concept, while an adversary tries to predict the concept.
We extend their work by deriving a class of general maximin games based on kernelization that largely maintains the tractability of the linear approach.
Our kernelized method performs concept erasure in a reproducing kernel Hilbert space, which may have a much higher dimensionality \citep{smola1998learning} and correspond to a non-linear subspace of the original representation space. 

 Empirically, we experiment with gender erasure from GloVe and BERT representations.
 We show that a kernelized adversary can classify the gender of the representations with over 99\% accuracy if $r(\cdot)$ is taken to be a linear function.
 This gives us concrete evidence that gender is indeed encoded non-linearly in the representations. 
 We further find that solving our kernelized maximin game yields an erasure function $r(\cdot)$ that protects against an adversary that shares the same kernel.
 However, we also find that it is difficult to protect against \emph{all} kernelized adversaries at once: Information removed by one kernel type can be recovered by adversaries using other kernel types. 
 That is, the gender concept is not exclusively encoded in a space that corresponds to any one kernel.
 This suggests that non-linear concept erasure is very much an open problem.\looseness=-1

\section{Linear Concept Erasure}
\label{sec:formulation}
We provide an overview of the linear maximin formulation before we introduce its kernelization.

\subsection{Notation}
Let $\dataset = \{(\xx_n, y_n)\}_{n=1}^N$ be a dataset of $N$ concept--representation pairs, where the labels $y_n$ represent the concept to be neutralized.
The goal of linear concept erasure is to learn a \emph{linear} erasure function $r(\cdot)$ from $\dataset$ such that it is impossible to predict $y_n$ from the modified representations $r(\xx_n)$.
We focus on classification, where the concept labels $y_n$ are derived from a finite set $\{1, \dots, V\}$ of $V$ discrete values, and the representations $\xx_n \in \R^D$ are $D$-dimensional real column vectors.
To predict the concepts labels $y_n$ from the representations $\xx_n$, we make use of (and later kernelize) classifiers that are \defn{linear models}, i.e., classifiers of the form $\ynhat = \vtheta^{\top} \xx_n$ where $\vtheta \in \Theta \subseteq \R^D$ is a column vector of parameters that lives in a compact set $\Theta$.
We also consider an arbitrary \defn{loss function} $\loss(\cdot, \cdot) \geq 0$ where $\loss(y_n, \ynhat)$ tells us how close the prediction $\ynhat$ is to $y_n$.\looseness=-1

Using this notation, linear concept erasure is realized by identifying a linear concept subspace whose neutralization, achieved through an orthogonal projection matrix, prevents the classifier from predicting the concept.
We define $\projk$ as the set of all $D \times D$ orthogonal projection matrices that neutralize a rank $k$ subspace.
More formally, we have that $\PP \in \projk \iff \PP=\II_D - \WW^{\top}\WW, \WW \in \R^{\targetdimension \times D}, \WW^{\top} = \II_{\targetdimension}$, where $\II_{\targetdimension}$ denotes the $\targetdimension \times \targetdimension$ identity matrix and $\II_D$ denotes the $D \times D$ identity matrix. 
We say that the matrix $\PP$ neutralizes the $\targetdimension$-dimensional rowspace of $\WW$.
\subsection{A Linear Maximin Game}
Following this formalization, it is natural to define a maximin game \citep{von1947theory} between a projection matrix   that aims to remove the concept subspace, and a linear model parameterized by $\vtheta$ that aims to recover it:
\begin{equation}
    \max_{{\PP \in \projk}}\,\min_{\vtheta \in \Theta}\,\sum_{n=1}^N \loss\Big(\yn,\, \vtheta^{\top} \PP\xx_n\Big). \label{eq:main-game}
\end{equation}
\Cref{eq:main-game} is a sequential game.
The first player chooses an orthogonal projection matrix $\PP$ to hinder the prediction of $\yn$ from $\xx_n$.
In response, the second player chooses the parameters of a linear model $\vtheta$, with knowledge of $\PP$, to best predict $\yn$ from $\PP \xx_n$.
This is a special case of the general adversarial framework \citep{DBLP:journals/corr/GoodfellowPMXWOCB14}.
However, in our case, the predictor $\vtheta$ does \emph{not} interact with the original input $\xx_n$.
Instead, the classifier attempts to predict the concept label $y_n$ from $\PP \xx_n$. 
We now give a concrete example of the linear concept: When we have a binary logistic loss ($V=2$) and a $\targetdimension$-dimensional neutralized subspace, the game takes the following form:\looseness=-1
\begin{equation}
\begin{aligned}
   \max_{\PP \in \projk}& \,\min_{\vtheta \in \Theta} \,\, - \Bigg( \sum_{n=1}^N \yn \log \frac{\exp \vtheta^{\top} \PP \xx_n} {1 + \exp \vtheta^{\top}\PP \xx_n} \\
   & + \Big (1-\yn \Big) \log \frac{1} {1 + \exp \vtheta^{\top}\PP \xx_n}  \Bigg).
\end{aligned}
\end{equation}

In this paper, we focus on $\targetdimension=1$, i.e., we aim to identify a 1-dimensional concept subspace.\footnote{This construction can be extended to subspaces of dimension $k > 1$ using a version of the Gram--Schmidt process in the RKHS.}
Every $D\times D$ orthogonal projection matrix of rank $D-1$ can be written as follows 
\begin{equation}
\PP^{\bot}_{\vw} \defeq \II - 
\frac{\vw \vw^{\top}}{\vw^{\top} \vw},
\end{equation}
itself a differentiable function of $\vw$.
This results in the following game:
\begin{equation}
   \!\!\!\!\max_{\vw \in \R^D} \min_{\vtheta \in \Theta}\,\sum_{n=1}^N \loss\left(\yn,\, \vtheta^{\top} \PP^{\bot}_{\vw} \xx_n\right). \label{eq:relaxed-game}
\end{equation}\ryan{Why did we switch form $\PP$ to $\PP^{\top}$}

\section{Non-linear Concept Erasure} \label{sec:non-linearerasure}
It has often been shown that human-interpretable concepts, and in particular gender, are encoded non-linearly in the representation space \citep{gonen-goldberg-2019-lipstick,ravfogel-etal-2020-null}.
However, prior work on concept erasure, e.g., the method discussed in \Cref{sec:formulation}, assumes that concepts are encoded linearly.
Our goal is to extend the game  defined in \cref{eq:relaxed-game} to be able to neutralize non-linearly encoded concepts while preserving the relative tractability and interpretability of the linear methods. 
A natural manner through which we can achieve these goals is by kernelization \citep{shawe2004kernel,hofmann2008kernel}.\looseness=-1

The underlying assumption motivating kernel methods is that the features needed for the task live in a reproducing kernel Hilbert space \citep[RKHS;][]{aronszajn50reproducing}.
At an intuitive level, an RKHS allows us to extend some of results from linear algebra to potentially infinite-dimensional inner prodcut spaces \citep{canu2006kernel}.
The main technical contribution of this paper is the derivation of a kernelized version of the linear adversarial game presented in \cref{eq:relaxed-game}. We perform this derivation in this section after providing some background on kernel methods. We show that the resulting kernelized game is both non-convex and too computationally heavy to solve directly.
Thus, we introduce a \nystrom\ approximation that results in an efficient and light-weight formulation, given in \cref{eq:relaxed-game-nystrom}, that is still able to isolate concepts encoded non-linearly.\looseness=-1

\subsection{Background: Kernel Methods}
Kernel methods are based on reproducing kernel Hilbert spaces (RKHS).
Without going into the technical details, an RKHS is a space of ``nice'' functions equipped with a kernel \citep{yosida2012functional}.
A \defn{kernel} $\kernel(\cdot, \cdot) \geq 0$ is a similarity measure that generalizes the notion of a positive definite matrix. 
When we have a kernel over an RKHS, the kernel corresponds to the dot product of a feature map, i.e., $\kernel(\xx, \yy) = \vphi(\xx)^{\top} \vphi(\yy)$, for feature map $\vphi$.
This insight gives us a natural manner to construct kernels.
For instance, the linear kernel $\kernel(\xx, \yy) = \xx^{\top}\yy$ corresponds to the standard dot product in Euclidean space. The degree-2 polynomial kernel $\kernel(\xx,\yy) = {(\gamma \xx^{\top} \yy + \alpha)}^2$ corresponds to a dot product in a six-dimensional feature space if $\xx,\yy \in \mathbb{R}^2$ where $\gamma \in \R$ is a hyperparameter.\looseness=-1

Kernels are more general than positive definite matrices in that they can exist in infinite-dimensional spaces. For example, the Gaussian kernel $\exp\left(-\gamma ||\xx-\yy||_2^2\right)$ is infinite-dimensional, where, again, $\gamma \geq 0$ is a hyperparameter.\footnote{One way to see this is to show a Gram matrix of any size $N$ will always have full rank.}
However, for any finite set of points $\{\xx_1, \dots, \xx_N \}$, we can construct a \defn{Gram matrix} $\K \in \R^{N \times N}$ where $\Kmn = \kernel(\xx_n, \xx_m)$ encodes the similarity between $\xx_n$ and $\xx_m$.
The matrix $\K$ is guaranteed to be positive definite.
Kernels are useful because they allow us to implicitly learn functions in a potentially infinite-dimensional RKHS without materializing that space.\looseness=-1

\subsection{A Kernelized Maximin Game}
Inspection of the maximin game in \cref{eq:main-game} reveals that both the adversary and the predictor interact with the input only via an inner product. Thus, the game can be kernelized by replacing the inner product with a kernel operation. 
We establish this kernelization by first proving the following representer theorem-like lemma, which shows that $\vw$ and $\vtheta$, now considered to live an RKHS $\hilbert$, can be written in terms of spans of the projected training set.\looseness=-1

\begin{restatable}{lemma}{representer}

\label{lemm:maximin} (Maximin Game Representer Theorem)
Let $\hilbert$ be a reproducing kernel Hilbert space with canonical feature map $\vphi \colon \R^D \rightarrow \hilbert$, i.e., $\vphi(\vx) = \kernel(\vx, \cdot)$.
Consider the following maximin game: 
\begin{align}\label{eq:kernel-game-alternative}
   &\max_{\vw \in \hilbert} \min_{
   \vtheta \in \hilbert} \sum_{n=1}^N \loss\left(\yn,\, \langle \vtheta, \projcomp\, \phx \rangle \right).
\end{align}
where $\projcomp \colon \hilbert \rightarrow \hilbert$ is the operator that projects onto the orthogonal complement of $\vw$.
For every attained local optimum $(\vtheta^\star$, $\vw^\star)$ of \Cref{eq:kernel-game-alternative}, there is another local optimum $(\vtheta_U^\star$, $\vw_U^\star)$ with the same value as $(\vtheta^\star$, $\vw^\star)$ in $U \defeq \spanmat \Big\{\vphi(\vx_1), \hdots, \vphi(\vx_N)\Big\}$, the span of the training data.\footnote{Traditionally, a representer theorem has an extra regularization term on the function estimator that ensures the optimum is attained and that it is global \citep{kimeldorf1970correspondence}. We do not employ such a regularizer and, thus, have a weaker result.}\looseness=-1

\end{restatable}
\begin{proof}
See \Cref{appdx:repr} for the proof.
\end{proof}

Having expressed $(\vw,\vtheta)$ as a function of the training data, we can proceed to proving the kernelization of the adversarial game.

\begin{restatable}{lemma}{kernelization}
\label{lemma:kernel_game_1}
Let $\hilbert$ be a reproducing kernel Hilbert space with canonical feature map $\vphi$, and let $\phz$ be a point in $\hilbert$.
Next, let  $\vw = \sum_{n=1}^N \alpha_{n} \phxn$ and $\vtheta = \sum_{n=1}^N \beta_n \phxn$ be points in the reproducing kernel Hilbert space.
Now, let $\phzproj$ be the orthogonal projection of $\phz$ onto the orthogonal complement of the subspace spanned by $\vw$.
Then, we have \looseness=-1
\begin{equation}
\begin{aligned}
\thetaphi 
= \sum_{m=1}^N &\beta_m \bigg(\kernel(\xxm, \zz) \\
&-\frac{\valpha^{\top} \Km(\zz) \valpha}{\valpha^{\top} \K \valpha} \bigg),
   \label{eq:prediction-kernel-space}
\end{aligned}
\end{equation}
where $\K^{(m)}_{ij}(\zz) \defeq \kernel(\xxi, \zz) \kernel(\xxm, \xxj)$. 
\end{restatable}
\begin{proof}
See \cref{app:kernelization} for the proof.
\end{proof}

\Cref{lemma:kernel_game_1} suggests the following form of the kernelized game given in \Cref{eq:kernel-game-alternative}:
\begin{equation}\label{eq:game-kernel2}
\begin{aligned}
\max_{\valpha \in \R^N}\, \min_{\vbeta \in \R^N}\,\sum_{n=1}^N &\loss\Bigg(\yn,\sum_{m=1}^N \beta_m \bigg(\kernel(\xxm, \zz_n)
\\
&-\frac{\valpha^{\top} \Kmn \valpha}{\valpha^{\top} K \valpha} \bigg)\Bigg),
\end{aligned}
\end{equation}
where we define $\K^{(m,n)}_{ij}(\zz) \defeq \kernel(\xxi, \zz_n) \kernel(\xxm, \xxj)$.
In contrast to \Cref{eq:kernel-game-alternative}, all computations in \Cref{eq:game-kernel2} are in Euclidean space. 
\begin{theorem}
The reproducing kernel Hilbert space game \Cref{eq:kernel-game-alternative} attains the same local optima as \Cref{eq:game-kernel2}.
\end{theorem}
\begin{proof}
First, plug the result of \Cref{lemma:kernel_game_1} into \Cref{eq:kernel-game-alternative}.
Optimality follows by \Cref{lemm:maximin}, the representer lemma.
\end{proof}
\noindent Now, we turn to the runtime of the game.
\begin{proposition}\label{prop:naive-runtime}
The objective in \Cref{eq:game-kernel2} can be computed in $\bigO(N^4)$ time.
\end{proposition}
\begin{proof}
Assuming that $\kernel(\cdot, \cdot)$ may be computed in $\bigO(1)$, computing $\Kmn$ takes $\bigO(N^2)$ time.
We have to do $\bigO(N^2)$ such computations, which results in $\bigO(N^4)$ time.
Note that $K$ may be pre-computed once in $\bigO(N^2)$ time. 
Thus, $\Kmn$ is the bottleneck, so the whole algorithm takes $\bigO(N^4)$ time.\looseness=-1
\end{proof}

There are two problems with naïvely using the formulation in \Cref{eq:game-kernel2}.
First, and similarly to \Cref{eq:main-game}, directly optimizing \Cref{eq:game-kernel2} is difficult because the parameters $\valpha$ implicitly define an orthogonal projection matrix of rank 1 and, thus, we face the sample problems discussed in \citet{rlace}.
Second, the evaluation time of \Cref{eq:game-kernel2} is $\bigO\left(N^4\right)$, as argued in \Cref{prop:naive-runtime}, which makes using a training set larger than a few hundred examples infeasible.
We solve both of these computational issues
with the Nystr{\"o}m approximation.

\subsection{The Nystr\"{o}m Approximation}
\label{sec:nystrom}
The general idea behind the Nystr\"{o}m method \citep{nystrom1930praktische} is to calculate a low-rank approximation of the kernel matrix.
It is a commonly used technique for improving the runtime of kernel methods \citep{williams2001using}.\looseness=-1

\subsubsection{Relaxing the Objective}

Consider the Gram matrix $\K \in \R^{N \times N}$ where $\Kmn = \kernel(\xxm, \xxn)$, $\xxm$ is the $m^{\text{th}}$ training representation, and $\xxn$ is the $n^{\text{th}}$ training representation.
We start with the eigendecomposition of $\K = \UU \SSigma \UU^{\top} = \UU \sqrt{\SSigma} \sqrt{\SSigma} \UU^{\top}$, which can be computed in $\bigO\left(N^3\right)$ time \citep{golub2013matrix}.
We are justified in taking the square root of $\SSigma$ because $\K$ is necessarily positive definite.
Now, we define an approximate feature map for observations $\xxn$ in the training data using the eigenvalues and vectors:\looseness=-1
\begin{equation}\label{eq:new-features}
\phxnyst \defeq \left(\UU \sqrt{\SSigma}\right)_n 
\end{equation}
To compute the features for representations $\xx$ \emph{not} in the training data, we use the following: 
\begin{equation}
\vphitilde\left(\xx \right) \defeq \sum_{n=1}^N \kernel(\xx, \xxn) \widetilde{\vphi}(\xx_n)
\end{equation}
which is an average, weighted by the kernel $\kappa(\cdot, \cdot)$, of the features obtained during training. Plugging \Cref{eq:new-features} into \Cref{eq:relaxed-game} yields: 
\begin{equation} 
  \max_{\vw \in \R^D}\, \min_{\vtheta \in \Theta}\,\sum_{n=1}^N \loss\Big(\yn,\, \langle \vtheta, \PP_{\vw}^{\bot} \phxnyst \rangle \Big) 
   \label{eq:relaxed-game-nystrom}
\end{equation}
which is identical to the linear game in \Cref{eq:relaxed-game}, except that it uses the transformed features $\phxnyst$ to approximate the true feature map $\vphi$.

While neither the game in \Cref{eq:relaxed-game} nor the version of  \Cref{eq:relaxed-game} with the \nystrom\ approximation given in \cref{eq:relaxed-game-nystrom} is convex--concave, we are still able to approximately solve the game in practice.
Specifically, we consider the convex relaxation of the set of rank-$\targetdimension$ orthogonal projection matrices given by \citet{rlace}.
Then, after applying the relaxation, we apply a gradient-based optimization procedure. 
The runtime of this procedure for this game over $T$ epochs is now $\bigO\left(TN^2 + N^3\right)$.\footnote{Calculating the \nystrom\ features entails computing the eigendecomposition of the Gram matrix, which is $\bigO\left(N^3\right)$. Then, we calculate the objective over all $N$ examples for $T$ epochs, which is $\bigO\left(TN^2\right)$.} 
While this is an improvement over $\bigO\left(TN^4\right)$, it is still not fast enough to be of practical use.

\subsubsection{Improving the Runtime}
The runtime bottleneck of the game in \Cref{eq:relaxed-game-nystrom} is the $\bigO\left(N^3\right)$ time it takes to compute the eigendecomposition of the Gram matrix $K$. 
Under the assumption that $\rank\left(\K\right) = L$, we can improve this bound to
$\bigO\left(L^3 + L^2 N \right)$ \citep{10.5555/1046920.1194916}.
For the case that $L \ll N$, this is a substantial improvement. 
Moreover, it implies that the approximation feature map $\vphitilde(\xx_n) \in \R^L$. 
After $T$ steps of optimization, the runtime is now $\bigO\left(TL^2 + L^3 + L^2 N \right)$, which is fast enough to be useful in practice.
A natural question to ask is what happens if we apply the \nystrom\ approximation, thereby assuming $\rank\left(\K\right) = L$, but in practice $\rank\left(\K\right) > L$? 
In this case, we are effectively computing a low-rank approximation of the kernel matrix.
Several bounds on the accuracy of this approximation have been proven in the literature \citep{10.5555/1046920.1194916, jin2011improved, nemtsov2016matrix}; we refer the reader to these works for more details on the approximation error.\looseness=-1

\subsection{Pre-image Mapping}
\label{sec:pre-image}
After solving the game in \cref{eq:relaxed-game-nystrom}, we obtain a projection matrix $\PP$ that neutralizes the concept \emph{inside} an (approximated) RKHS. In other words, we have a function $r(\cdot)$ that prevents a classifier from predicting a concept label from the representation $\vphitilde(\xx_n)$.
However, for many applications, we want a version of the input $\xx_n$ \emph{in the original space} with the concept neutralized, i.e., we want $r(\xx_n)$.
Neutralization in the original space requires solving the \textbf{pre-image problem} \citep{mika1998kernel}.
In the case of Nystr{\"o}m features, we seek a mapping from feature space back to the original Euclidean space, i.e., a mapping from $\R^L$ to $\R^D$, that projects the neutralized features back into the input space.\looseness=-1

In practice, this task can also be performed via a mapping from $\R^D \mapsto \R^D$, which learns to reproduce in the input space the transformation that $\PP$ performs in the RKHS. We choose the latter approach, and train a multilayer perceptron (MLP) $\flambda(\cdot) \colon \R^D \rightarrow \R^D$.\footnote{This mapping may not always be a function because there may be more than one point in the RKHS that corresponds to a point in the original space.} To estimate the parameters $\lambda$ of the MLP $f_{\lambda}(\cdot)$, we optimize a two-termed objective over all points $\xx_n$:
\begin{align}\label{eq:pre-imageobj}
    \argmin_{\lambda \in \Lambda} &{\left|\left|\PP \phxnyst - \boldsymbol{\phnyst}(\flambda(\xx_n))\right|\right|_2^2}  \\ 
    & \quad + \left|\left|(\II-\PP) \phnyst(\flambda(\xx_n))\right|\right|_2^2 \nonumber
\end{align}
where $\Lambda$ is the parameter space.
The first term encourages $\flambda(\cdot)$ to perform the same transformation in the input space as $\PP$ does in the RKHS. The second term ensures that $\PP$ has no effect on the RKHS features computed on the neutralized $\flambda(\xx_n)$.\looseness=-1

\section{Experimental Setup}

In \Cref{sec:non-linearerasure}, we established an algorithm that allows us to attempt kernelized concept erasure of non-linearly encoded concepts. To summarize, this method requires first solving the game in \Cref{eq:relaxed-game-nystrom} for a chosen neutralizing kernel, then training a pre-image network according to \cref{eq:pre-imageobj} to obtain neutralized representations in the input space. \looseness=-1

We hypothesize that a non-linearly encoded concept can be exhaustively removed after mapping into the right RKHS. In order for this to hold, the neutralized representations must satisfy two conditions: Adversaries using non-linear classifiers should not be able to predict the erased concept from these representations, and these representations should preserve all other information encoded in them prior to erasure. With binary gender as our non-linearly encoded concept, we conduct several experiments testing both conditions. Before presenting our results, we lay out our experimental setup (see \cref{app:exp-setting} for more details). 

\paragraph{Data.} We run our main experiments on the identification and erasure of binary gender in static GloVe representations
\citep{DBLP:conf/emnlp/PenningtonSM14}. We focus on \citeposs{ravfogel-etal-2020-null} dataset, where word representations are coupled with binary labels indicating whether they are male-biased or female-biased. As a preprocessing step, we normalize the GloVe representations to have unit norm. For an extrinsic evaluation of our method on a main task (profession prediction), we use the Bias-in-Bios dataset of \citet{DBLP:journals/corr/abs-1901-09451}, which consists of a large set of short biographies annotated for both gender and race. Following \citet{rlace}, we embed each biography using the {\small \textsf{[CLS]}} representation of pre-trained BERT.

\paragraph{Kernels.}
We consider the following kernels:
\begin{itemize}
    \item \poly\:: $\kernel(\xx,\yy)={(\gamma \xx^{\top}\yy + \alpha)}^{\degree}$
    \item \rbf\:: $\kernel(\xx,\yy)=\exp\left(-\gamma ||\xx-\yy||_2^2\right)$
    \item \laplace\:: $\kernel(\xx,\yy)=\exp\left(-\gamma ||\xx-\yy||_1\right)$
    \item \linear\:: $\kernel(\xx,\yy)= \xx^{\top}\yy $
    \item \sigmoid\:: $\kernel(\xx,\yy)=\mathrm{tanh}(\gamma \xx^{\top} \yy + \alpha)$
    \item \multiple\:: a convex combination of the above kernels. We consider the following two methods for combining kernels:
    \begin{itemize}
        \item \easymkl\:: a convex combination learned with the EasyMKL algorithm \citep{aiolli2015easymkl} targeted for gender prediction.\footnote{We use the implementation of \href{https://github.com/IvanoLauriola/MKLpy}{MKLpy} with the default parameters.}
        \item \uniformmk\:: a uniform combination of all kernels.
    \end{itemize}
\end{itemize}
\noindent We experiment with different values for the hyperparameters $\gamma > 0$, $\alpha > 0$ and $\degree > 0$ (see \Cref{app:exp-setting} for details). We use $L=1024$-dimensional vectors for the \nystrom\ approximation. 
\paragraph{Reported metrics.} Each result is reported as the mean $\pm$ standard deviation, computed across four runs of the experiment with random restarts.\looseness=-1

\paragraph{Solving the maximin game.} 
We solve the relaxed adversarial game given in \Cref{eq:relaxed-game-nystrom} by alternate gradient-based optimization \citep{goodfellow2016deep}.
Concretely, we alternate between updating the predictor's parameters $\vtheta$ and the projection matrix $\PP$.
Updates to $\vtheta$ are performed with gradient descent, and updates to $\PP$ are performed with gradient \emph{ascent}, including a projection onto the Fantope to ensure that the constraint is met.
For the Fantope projection step, we use \citeposs{vu2013fantope} algorithm, the details of which are restated in \citet{rlace}. 

\paragraph{Pre-image calculation.} 
As our pre-image network $f_{\lambda}(\cdot)$, we use an MLP with two hidden layers of sizes 512 and 300, respectively.
We use layer normalization after each hidden layer and ReLU activations. 
See \cref{app:pipeline-eval} for basic empirical validation of the pre-image calculation procedure. 

\section{Effect on Concept Encoding}

In this section, we pose the \defn{exhaustive RKHS hypothesis}: The hypothesis that binary gender can be \emph{exhaustively} removed when the representations are mapped into the right RKHS. That is, there exists a unique kernel such that, for any choice of non-linear predictor, the adversary cannot recover gender information from the pre-image representations obtained via our special kernel.
As a baseline, we note that gender prediction accuracy on the original representations, prior to any intervention, is above 99\% with \emph{every} kernel, including the linear kernel. 
This means that the gender concept is linearly separable in the original input space.
In this context, we conduct the following experiments on gender neutralization.\looseness=-1 
\paragraph{Same adversary.} We start by calculating the neutralized pre-images for each kernel type, and then apply the same kernel adversary to recover gender information. This experiment tests whether we can protect against the same kernel adversary.\looseness=-1 

\paragraph{Transfer between kernels.} To directly test the exhaustive RKHS hypothesis, we calculate the neutralized pre-image representations with respect to a neutralizing kernel, and then use a different adversarial kernel to recover gender. For instance, we calculate neutralized pre-image representations with respect to a polynomial kernel, and then predict gender using a Laplace kernel, or a polynomial kernel with different hyperparameters.\looseness=-1 

\paragraph{Biased associations.} Gender can manifest itself in the pre-image representations via biased associations, even when gender is neutralized according to our adversarial test. To assess the impact of our intervention on this notion of gender encoding, we run the WEAT test \citep{DBLP:journals/corr/IslamBN16} on the neutralized pre-image representations.\looseness=-1

\begin{table}[]
\centering
\begin{tabular}{ll}
\toprule
     Type &        Accuracy \\
\midrule
     \poly & 0.59 $\pm$ 0.15 \\
      \rbf & 0.69 $\pm$ 0.16 \\
  \laplace & 0.75 $\pm$ 0.11 \\
   \linear & 0.54 $\pm$ 0.02 \\
  \sigmoid & 0.49 $\pm$ 0.00 \\
  \easymkl & 0.69 $\pm$ 0.01 \\
\uniformmk & 0.49 $\pm$ 0.00 \\
\bottomrule
\end{tabular}

\caption{Gender prediction accuracy from the neutralized pre-image representations when using the same kernel for neutralization and recovery. Numbers are averages over hyperparameters of each kernel and over four randomized runs of the experiment.}
\label{tab:pre-image-pred-gender}
\end{table}

\begin{table*}[h]
\centering
\scalebox{0.75}{
\begin{tabular}{lllllllll}
\toprule
{} &             \poly &              \rbf &          \laplace &           \linear &          \sigmoid &          \easymkl &        \uniformmk &              \mlp \\
\midrule
\poly      &  0.98 $\pm$ 0.00 &  0.96 $\pm$ 0.00 &  0.93 $\pm$ 0.00 &  0.55 $\pm$ 0.01 &  0.49 $\pm$ 0.00 &  0.98 $\pm$ 0.00 &  0.49 $\pm$ 0.00 &  0.97 $\pm$ 0.00 \\
\rbf       &  0.98 $\pm$ 0.00 &  0.96 $\pm$ 0.00 &  0.93 $\pm$ 0.00 &  0.59 $\pm$ 0.02 &  0.49 $\pm$ 0.00 &  0.98 $\pm$ 0.00 &  0.49 $\pm$ 0.00 &  0.97 $\pm$ 0.00 \\
\laplace   &  0.98 $\pm$ 0.00 &  0.96 $\pm$ 0.00 &  0.94 $\pm$ 0.00 &  0.61 $\pm$ 0.01 &  0.49 $\pm$ 0.00 &  0.98 $\pm$ 0.00 &  0.59 $\pm$ 0.02 &  0.97 $\pm$ 0.00 \\
\linear    &  0.98 $\pm$ 0.00 &  0.96 $\pm$ 0.00 &  0.93 $\pm$ 0.00 &  0.54 $\pm$ 0.02 &  0.49 $\pm$ 0.00 &  0.98 $\pm$ 0.00 &  0.49 $\pm$ 0.00 &  0.97 $\pm$ 0.00 \\
\sigmoid   &  0.98 $\pm$ 0.00 &  0.93 $\pm$ 0.01 &  0.89 $\pm$ 0.01 &  0.65 $\pm$ 0.03 &  0.49 $\pm$ 0.00 &  0.97 $\pm$ 0.00 &  0.64 $\pm$ 0.03 &  0.97 $\pm$ 0.00 \\
\easymkl   &  0.98 $\pm$ 0.00 &  0.96 $\pm$ 0.00 &  0.94 $\pm$ 0.00 &  0.57 $\pm$ 0.03 &  0.49 $\pm$ 0.00 &  0.98 $\pm$ 0.00 &  0.49 $\pm$ 0.00 &  0.97 $\pm$ 0.00 \\
\uniformmk &  0.98 $\pm$ 0.00 &  0.96 $\pm$ 0.00 &  0.94 $\pm$ 0.01 &  0.58 $\pm$ 0.08 &  0.49 $\pm$ 0.00 &  0.98 $\pm$ 0.00 &  0.49 $\pm$ 0.00 &  0.97 $\pm$ 0.00 \\
\bottomrule
\end{tabular}

}
\caption{Evaluation of the neutralized pre-image representations when using a different kernel for neutralization and recovery. The neutralizing kernel is presented in the rows, and the kernel adversary used for recovery in the columns. Note that on the diagonal, the \emph{hyperparameters} of the same kernel family differ between the neutralizing kernel and the adversary that recovers gender form the pre-image representations. See \cref{app:trasnfer} for a breakdown by neutralizing kernel hyperparameters and for details on the hyperparameters of the kernel adversaries.}
\label{tab:pre-image-pred-avg}
\end{table*}

\subsection{Pre-image Recovery: Same Adversary} 

In \Cref{tab:pre-image-pred-gender}, we report average gender prediction accuracy on the neutralized pre-images for the case where the neutralizing and adversarial kernels are of the same type and share the same hyperparameters.\footnote{We use the \nystrom\ approximation for neutralization, but predict gender from the pre-image representations using SVM classifiers that use the true kernel.} Numbers are averages over the results of all hyperparameter values of each kernel.
See \cref{app:gender-results} for the full results. As can be seen, for most---but not all---kernels, we effectively hinder the ability of the non-linear classifier to predict gender.\looseness=-1

\subsection{Pre-image Recovery: Kernel Transfer} 
\label{sec:transfer}

In \Cref{tab:pre-image-pred-avg}, we report average gender prediction accuracy on the neutralized pre-images for the case where the neutralizing and adversarial kernels are different. Numbers are averages over several different hyperparameter settings for the neutralizing kernel, while the adversarial kernel hyperparameters are fixed as detailed in \Cref{app:trasnfer}.
Also, see the appendix for a full breakdown by neutralizing kernel hyperparameters.
\Cref{tab:pre-image-pred-avg} allows us to test the exhaustive RKHS hypothesis. Under this hypothesis, we would expect to see that for at least one neutralizing kernel, no non-linear classifers are able to accurately recover gender. For a thorough test of the hypothesis, we introduce an MLP with a single hidden layer as an additional adversary.\looseness=-1 

Remarkably, we observe a complete lack of generalization of our concept erasure intervention to other types of non-linear predictors. In particular, no neutralizing kernel significantly hinders the ability of an MLP with a single hidden layer to predict gender: the MLP always recovers the gender labels with an accuracy of 97\%. Furthermore, concept erasure does not transfer between different kernel types, and even between kernels of the same family with different hyperparameter settings. For instance, when using a polynomial neutralizing kernel, we protect against a polynomial adversarial kernel with the same parameters (mean accuracy of 59\% in \Cref{tab:pre-image-pred-gender}), but not against a polynomial adversarial kernel with different hyperparameters (98\% mean accuracy for \poly\ in \Cref{tab:pre-image-pred-avg}).

We do see transfer to sigmoid, and---to a lesser degree---linear kernel adversaries, with a classification accuracy of 54--65\% for the linear adversary, and 49\% for the sigmoid adversary. Surprisingly, the sigmoid kernel seems weaker than the linear kernel, achieving a near-random accuracy of 49\%. The results do not show evidence of a proper hierarchy in the expressiveness of the different kernels, and convex combinations of the different kernels---either learned (\easymkl) or uniform (\uniformmk)---do not provide better protection than individual kernels.\looseness=-1

In short, while we are able to effectively protect against the same kernel, transfer to different kernels is non-existent. This result does not support the exhaustive RKHS hypothesis: We do not find a single RKHS that exhaustively encodes the binary gender concept.

\begin{table}[]
\centering
\begin{tabular}{lllllll}
\toprule
  Kernel &           \texttt{WEAT's} $d$ &                  \texttt{WEAT's} $p$-value \\
\midrule
     \poly & 0.74 $\pm$ 0.01 & 0.08 $\pm$ 0.00  \\
      \rbf & 0.74 $\pm$ 0.00 & 0.08 $\pm$ 0.00  \\
  \laplace & 0.71 $\pm$ 0.03 & 0.09 $\pm$ 0.01 \\
   \linear & 0.74 $\pm$ 0.00 & 0.08 $\pm$ 0.00 \\
  \sigmoid & 0.75 $\pm$ 0.02 & 0.08 $\pm$ 0.01 \\
  EasyMKL & 0.73 $\pm$ 0.00 & 0.08 $\pm$ 0.00  \\
UniformMK & 0.73 $\pm$ 0.00 & 0.08 $\pm$ 0.00 \\
    \hline
 Original & 1.56 & 0.000 \\
\bottomrule
\end{tabular}
\caption{WEAT results on pre-image representations. Numbers are averages over hyperparameters of each kernel.
}
\label{tab:pre-image-pred}
\end{table}

\subsection{Effect on Gendered Word Associations}

In the case where the concept of gender cannot be recovered by an adversary, binary gender could still manifest itself in more subtle ways. For instance, it may be the case that the names of gender-biased professions, such as STEM fields, are closer in representation space to male-associated words than to female-associated words. We aim to measure the extent to which our neutralized pre-image representations exhibit this measure of gender bias.

\paragraph{Evaluation.} 
To quantify bias associations, \citet{DBLP:journals/corr/IslamBN16} propose the WEAT word association test. This test measures WEAT's $\degree$, a statistic that quantifies the difference in similarity between two sets of gendered words (e.g., male first names and female first names) and two sets of potentially biased words (e.g., stereotypically male and stereotypically female professions).\footnote{It asks, for example, whether the degree to which male first names are closer to scientific terms versus artistic terms is significantly different from the same quantity calculated for female first names.} We repeat the experiments of \citet{gonen-goldberg-2019-lipstick} and \citet{ravfogel-etal-2020-null}.  Following \citet{gonen-goldberg-2019-lipstick}, we represent the male and female groups with names commonly associated with males and females, rather than with explicitly gendered words (e.g., pronouns). Three tests evaluate the association between name groups and i) career and family-related words; ii) art and mathematics-related words; and iii) names of artistic and scientific fields. Successful neutralization of gender would imply that these word groups are less closely associated in the pre-image representations.

\paragraph{Results.} In \Cref{tab:pre-image-pred}, we report the test statistic and the $p$-value for the third test using the names of scientific and artistic fields to represent gender-biased words.\footnote{See \Cref{app:weat} for results of the other two biased word association tests.} For all kernels, we observe a significant drop in the test statistic, from the original value of 1.56 to around 0.74. This suggests that the intervention significantly decreases the association between female and male names and stereotypically biased words.
 Notably, the reduction is similar for \emph{all} kernels, including the linear one. While non-linear erasure is more effective in neutralizing gender against adversarial recovery, linear and non-linear methods perform equally well according to this bias association test. This finding highlights the importance of measuring different manifestations of a concept when using concept neutralization as a bias mitigation method.\looseness=-1

\section{Negative Impact on the Representations}

Our method has shown a satisfactory ability to prevent the same kernel from recovering the concept. However, does erasure remove \emph{too much} information?\footnote{It is trivial to neutralize a concept by completely zeroing out the representations.} As previously stated, our intervention should erase a concept without altering any of the other information encoded in the original representations. In this section, we evaluate whether the non-gender related semantic content of the original representations is preserved in our neutralized pre-images. We do so via the following tests: i) an \textbf{intrinsic evaluation} of the semantic content of the neutralized pre-image word representation space, and
ii) an \textbf{extrinsic evaluation} of our method when applied to contextualized word representations for a profession prediction task, measuring the extent to which we hinder a model's ability to perform the main task.\looseness=-1

\subsection{Intrinsic Evaluation of Damage to Semantic Content}

To measure the influence of our method on the semantics encoded in the representation space, we use SimLex-999 \citep{DBLP:journals/coling/HillRK15}, an annotated dataset of word pairs with human similarity scores for each pair. First, we calculate the cosine similarity between the representations of each pair of words using the original representations. Then, we repeat this calculation for each type of kernel using the pre-image representations. Finally, we measure the correlation between the similarity of words in representation space and human similarity scores, before and after intervention. The original correlation is $0.400$, and it is left nearly unchanged by \emph{any} of the kernel interventions, yielding values between $0.387$ and $0.396$. To qualitatively demonstrate the absence of negative impact, we show in \cref{app:neighbors} that the nearest neighbors of randomly sampled words do not change significantly after gender erasure.\looseness=-1

\subsection{Extrinsic Evaluation on Contextualized Representations} 

The previous experiments focused on the influence of concept neutralization on uncontextualized representations. 
Here, we apply our concept neutralization method on contextualized BERT representations and assess its effect on profession prediction. We embed each biography in the dataset of \citet{DBLP:journals/corr/abs-1901-09451} using the {\small \textsf{[CLS]}} representation of pre-trained BERT, and apply our method using only the RBF kernel. After collecting the pre-image representations, we train a linear classifier on the main task of profession prediction.

\paragraph{Results.} Averaged over different hyperparameter settings of the RBF kernel, we achieve a profession prediction accuracy after neutralization of $74.19 \pm 0.056 \%$. 
For reference, prediction accuracy using the original BERT representations is $76.93\%$. This suggests that the pre-images still encode most of the profession information, which is largely orthogonal to the neutralized gender information. 

\section{Discussion}

We have demonstrated that in the case where the neutralizing kernel and the adversarial kernel are the same, we are able to neutralize a non-linearly encoded concept reasonably well. We have also shown that our method neutralizes gender in a comprehensive manner, without damaging the representation. However, this neutralization does not transfer to different non-linear adversaries, which are still able to recover gender.

While the lack of transfer to other non-linear predictors may seem surprising, one should keep in mind that changing the kernel type, or changing kernel hyperparameters, results in a different implicit feature mapping. Even the features defined by a linear kernel are not a proper subset of the features defined by a polynomial kernel of degree 2.\footnote{Features defined by a linear kernel can be a proper subset of the features defined by a polynomial kernel if there is a fixed coordinate in all examples prior to the kernel mapping.}
As such, removing the features which make the concept of interest linearly separable in one RKHS does not necessarily prevent a classifier parameterized by another kernel or an MLP from predicting the concept. 
In the context of gender erasure, these results suggest that protection against a diverse set of non-linear adversaries remains an open problem.

\section{Conclusion}
We propose a novel method for the identification and erasure of non-linearly encoded concepts in neural representations.
We first map the representations to an RKHS, before identifying and neutralizing the concept in that space.
We use our method to empirically assess the exhaustive RKHS hypothesis: We hypothesize that there exists a unique kernel that exhaustively identifies the concept of interest. We find that while we are able to protect against a kernel adversary of the same type, this protection does not transfer to different non-linear classifiers, thereby contradicting to the RKHS hypothesis. Exhaustive concept erasure and protection against a diverse set of non-linear adversaries remains an open problem.  

\section*{Limitations}
The empirical experiments in this work involve the removal of binary gender information from pre-trained representations. 
We note the fact that gender is a non-binary concept as a major limitation of our work. 
This task may have real-world applications, in particular relating to fairness. We would encourage readers to be careful when attempting to deploy methods such as the one discussed in this paper. Regardless of any proofs, one should carefully measure the effectiveness of the approach in the context in which it is to be deployed. Please consider, among other things, the exact data to be used, the fairness metrics under consideration, and the overall application.  

We urge practitioners not to regard this method as a solution to the problem of bias in neural models, but rather as a preliminary research effort toward mitigating certain aspects of the problem. 
Unavoidably, the datasets we use do not reflect all the subtle and implicit ways in which gender bias is manifested. As such, it is likely that different forms of bias still exist in the representations following the application of our method.

\section*{Ethical Concerns}
We do not foresee ethical concerns with this work.\looseness=-1

\section*{Acknowledgements}
The authors sincerely thank Clément Guerner for his thoughtful and comprehensive comments and revisions to the final version of this work. 
This project received funding from the European Research Council (ERC) under the European Union's Horizon 2020 research and innovation program, grant agreement No. 802774 (iEXTRACT). 
Ryan Cotterell acknowledges 
Google for support from the Research Scholar Program. 

\bibliography{anthology,main}
\bibliographystyle{acl_natbib}
\onecolumn 
\appendix

\section{Related Work}
\paragraph{Mitigation of Gender Bias.}
The identification of linear subspaces that encode binary gender has attracted considerable research interest \citep{bolukbasi2016man, gonen-goldberg-2019-lipstick, dev2019attenuating, ravfogel-etal-2020-null}. While bias mitigation is a central use case of concept erasure, concept subspaces have been applied to a number of tasks. Concept subspaces have been used to analyze the content of neural representations, e.g., for causal analysis \citep{elazar2021amnesic,ravfogel-etal-2021-counterfactual}, for analyzing the geometry of the representation space \citep{celikkanat2020controlling,gonen2020s, hernandez-andreas-2021-low}, and for concept-based interpretability \citep{kim2018interpretability}.

\paragraph{Kernelization of Linear Methods.}
The kernelization of linear machine learning algorithms is a common practice, and has many use cases, such as the kernelized perceptron \citep{aizerman1964theoretical} and kernel PCA \citep{scholkopf1997kernel}. \citet{white2021non} proposed a kernelization of a structural probe that extracts syntactic structure from neural representations. \citet{vargas-2020} proposed a kernelization of the PCA-based bias mitigation method of \citet{bolukbasi2016man}, and found that it does not improve on the linear mitigation procedure. Since the effectiveness of this method has been questioned \citep{gonen-goldberg-2019-lipstick}, we consider a more principled and well-motivated approach for the identification and neutralization of the concept subspace. \citet{sadeghi2019global} proposed a kernelization of an alternative, regression-based linear adversarial objective, which is not limited to \emph{orthogonal projections}. Our formulation is different in that it considers any linear model, and is restricted to the neutralization of linear subspaces via projection. This makes our method potentially less expressive, but more interpretable. 

\subsection{A Representer Lemma for Kernelized Maximin Games} \label{appdx:repr}

\representer*

\begin{proof}
For brevity, first notice we can re-express the objective as
\begin{align}\label{eq:game}
   &\max_{\vw, \in \hilbert} \min_{
   \vtheta \in \hilbert} \sum_{n=1}^N \loss\left(\yn,\, \Big\langle \vtheta, \left(\identity - \frac{{\vw\vw^\top}}{\vw^{\top} \vw} \right) \phx \Big\rangle \right).
\end{align}
We will show that both $\vw$ and $\vtheta$ can be expressed as a linear combination of terms from the training data without losing expressive power.
Now, decompose $\vw$ as follows: $\vw = \vw_{U} + \vw_{\perp U}$, where we represent $\vw$ as the sum of $\vw$ projected onto $U$ and onto its orthogonal complement $U_{\perp}$.
Now, note that for any element of the training data $\vphi(\vx_n)$, we have
\newcommand{\vwU}{\vw_U}
\newcommand{\vwUperp}{\vw_{\perp U}}

\begin{subequations}
\begin{align}
\left(\identity - \frac{\vw\vw^{\top}}{\vw^{\top}{\vw}}\right) \vphi(\vx_n) &= \left(\identity - \frac{(\vwU + \vwUperp) (\vwU + \vwUperp)^{\top}}{\vw^{\top}{\vw}}\right) \vphi(\vx_n) \\
&= \left(\identity - \frac{\vwU\vwU^{\top}}{\vw^{\top}{\vw}} - \underbrace{\frac{\vwU\vwUperp^{\top}}{\vw^{\top}{\vw}}}_{=0} - \underbrace{\frac{\vwUperp\vwU^{\top}}{\vw^{\top}{\vw}}}_{=0} - \frac{\vwUperp\vwUperp^{\top}}{\vw^{\top}{\vw}} \right)\vphi(\vx_n) \\
&= \left(\identity - \frac{\vw_U\vw_U^{\top}}{\vw^{\top}{\vw}} - \frac{\vw_{\perp U}\vw_{\perp U}^{\top}}{\vw^{\top}{\vw}}\right) \vphi(\vx_n)  \\
&= \vphi(\vx_n) - \frac{\vw_U\vw_U^{\top}}{\vw^{\top}{\vw}} \vphi(\vx_n)- \underbrace{\frac{\vw_{\perp U}\vw_{\perp U}^{\top}}{\vw^{\top}{\vw}}}_{=0} \label{eq:cancel-step2} \vphi(\vx_n)   \\
&= \vphi(\vx_n) - \vw_U\frac{\vw_U^{\top}\vphi(\vx_n)}{\vw^{\top}{\vw}}.
\end{align}
\end{subequations}
However, we have that $\vphi(\vx_n) - \vw_U\frac{\vw_U^{\top}\vphi(\vx_n)}{\vw^{\top}{\vw}}$ is in $U$.
Likewise, we can decompose $\vtheta = \vtheta_U + \vtheta_{\perp U}$.
Further manipulation reveals
\begin{equation}
\Big\langle \vtheta, \left(\identity - {\vw\vw^\top} \right) \phx \Big\rangle = \vtheta_{U}^{\top}\vphi(\vx_n) - \vtheta_{U}^{\top}\vw_U\frac{\vw_U^{\top}\vphi(\vx_n)}{\vw^{\top}{\vw}}. 
\end{equation}\ryan{We are missing a denominator in 14.}
Thus, for any $\vtheta, \vw,   \in \hilbert$ there exists a $\vtheta_{U}, \vw_U \in U$ that yields the same value of the objective as $\vtheta, \vw$.   
Now, we can parameterize $\vtheta_U$ and $\vw_U$ as 
\begin{align}
\vw_U &= \sum_{n=1}^N \alpha_n \vphi(\vx_n) \\
\vtheta_U &= \sum_{n=1}^N \beta_n \vphi(\vx_n).
\end{align}
for real coefficients $\valpha \in \R^N$ and $\vbeta \in \R^N$.
We conclude that for any local optimum $\vtheta^\star, \vw^\star \in \hilbert$, the projection of $\vtheta^\star$ and $\vw^\star$ onto $U$ yields a local optimum with the same value.
\end{proof}

\subsection{Kernelization of the Maximin Game}
\label{app:kernelization}

We show that the game \cref{eq:main-game} can be kernelized for the case $k=1$, i.e., a setting where the matrix $\projk$ removes a one-dimensional subspace.
Specifically, we will show that the product $\langle \vtheta, \Pmat \phxn \rangle$ in \cref{eq:main-game} can be expressed as a function of the kernel $\kernel(\cdot, \cdot)$.

\kernelization*

\begin{proof}

The projection onto the orthogonal complement of $\vw = \sum_{n=1}^N \alpha_n \phx$ is defined as the following
\begin{equation}
\projcomp \defeq \identity - \frac{\Big(\sum_{n=1}^N{\alpha_n \phxi}\Big)\Big(\sum_{n=1}^N{\alpha_n \vphi(\xxn)^{\top}}\Big)} {\Big(\sum_{n=1}^N\alpha_n  \vphi(\xx_n)^{\top}\Big)\Big(\sum_{n=1}^N \alpha_n \vphi(\xx_n) \Big)}
\end{equation}
where $\identity$ is the identity operator.
Algebraic manipulation reveals
\allowdisplaybreaks
\begin{subequations}
\begin{align}
     \projcomp\phz &\defeq \identity \phz - \frac{\Big(\sumi{\alpha_n \phx}\Big)\Big(\sum_{n=1}^N{\alpha_{n} \phxnprime^{\top}}\Big)} {\Big(\sum_{n=1}^N \alpha_n \vphi(\xx_n)^{\top}\Big)\Big(\sum_{n=1}^N\alpha_{n} \vphi(\xx_n) \Big)}\, \phz  \\
    &= \phz - \frac{\sum_{n=1}^N \sum_{m=1}^N {\alpha_n \alpha_{m} \vphi(\xx_n) \vphi(\xx_m)^{\top}} \phz } {\sum_{n=1}^N \sum_{m=1}^N {\alpha_n \alpha_{m} \vphi(\xx_n)^{\top} \vphi(\xx_m)} } \\ 
    &= \phz - \frac{\sum_{n=1}^N \sum_{m=1}^N {\alpha_n \alpha_{m} \vphi(\xx_n) \kernel(\xx_m, \zz)} } {\sum_{n=1}^N \sum_{m=1}^N {\alpha_n \alpha_{m} \kernel(\xxn, \xxm)} } \\
    &= \phz - \underbrace{\left(\frac{\sum_{m=1}^N {\alpha_{m} \kernel(\xxm, \zz)}} {\valpha^{\top} K \valpha}\right)}_{\in \R} \sumi \alpha_n \phx \\
       &= \phz - \left(\frac{\sum_{m=1}^N {\alpha_{m} \kernel(\xxn, \zz)}} {\valpha^{\top} K \valpha}\right) \vw
    \label{projected}  \\
    &\defeq \phzproj
\end{align}
\end{subequations}
Now, consider an element of the reproducing kernel Hilbert space
\begin{equation}
    \vtheta = \sum_{n=1}^N \beta_n \vphi(\xxn) \label{eq:assum2}
\end{equation}
Further algebraic manipulation reveals
\allowdisplaybreaks
\begin{subequations}
\begin{align}
    \Big\langle \vtheta,&\phzproj \Big\rangle =  \Big\langle \vtheta, \phz - \left(\frac{\sumi {\alpha_n \kernel(\xxn, \zz)}} {\valpha^{\top} K \valpha}\right) \vw \Big\rangle \\ 
    &=  \Big\langle \sum_{m=1}^N \beta_{m} \vphi(\xxm), \phz - \left(\frac{\sum_{n=1}^N {\alpha_n \kernel(\xxn, \zz)}} {\valpha^{\top} K \valpha}\right) \vw \Big\rangle \\
   &= \sum_{m=1}^N\beta_{m} \kernel(\xxm, \zz) - \sum_{m=1}^N \beta_{m} \left(\frac{\sum_{n=1}^N {\alpha_n \kernel(\xxn, \zz)}} {\valpha^{\top} K \valpha}\right) \kernel(\xxm, \vw) \\
   &= \sum_{m=1}^N \beta_{m} \kernel(\xxm, \zz) - \frac{1}{\valpha^{\top} K \valpha} \sum_{n=1}^N \sum_{m=1}^N \alpha_n \beta_{m} \kernel(\xxn, \zz) \kernel(\xx_m, \vw) \\
      &= \sum_{m=1}^N \beta_{m} \kernel(\xxm, \zz) - \sum_{m=1}^N 
      \beta_m \frac{\left(\sum_{n=1}^N \alpha_n  \kernel(\xxn, \zz) \kernel(\xx_m, \vw)\right)}{\valpha^{\top} K \valpha} \\
          &= \sum_{m=1}^N \beta_{m} \kernel(\xxm, \zz) - \sum_{m=1}^N 
      \beta_m \frac{\left(\sum_{n=1}^N \alpha_n  \kernel(\xxn, \zz) \vphi(\xx_m)^{\top}\left(\sum_{n'=1}^N \alpha_{n'} \vphi(\xx_{n'})\right)\right)}{\valpha^{\top} K \valpha} \\
              &= \sum_{m=1}^N \beta_{m} \kernel(\xxm, \zz) - \sum_{m=1}^N 
      \beta_m \frac{\left(\sum_{n=1}^N \sum_{n'=1}^N \alpha_{n'} \alpha_n  \kernel(\xxn, \zz) \vphi(\xx_m)^{\top}\vphi(\xx_{n'})\right)}{\valpha^{\top} K \valpha} \\
                    &= \sum_{m=1}^N \beta_{m} \kernel(\xxm, \zz) - \sum_{m=1}^N 
      \beta_m \frac{\left(\sum_{n=1}^N \sum_{n'=1}^N \alpha_n  \alpha_{n'}  \kernel(\xxn, \zz) \kernel(\xx_m, \xx_{n'})\right)}{\valpha^{\top} K \valpha} \\
  &= \sum_{m=1}^N \beta_{m} \left(\kernel(\xxm, \zz) - \frac{\valpha^{\top} \Km(\zz)  \valpha}{\valpha^{\top} \K \valpha} \right)
   \label{eq:prediction-final}
\end{align}
\end{subequations}
where we define the following matrix component-wise: $\boldsymbol{K}^{(m)}_{ij}(\zz) \defeq \kernel(\xxi, \zz) \kernel(\xxm, \xxj)$.
\Cref{eq:prediction-final} can be evaluated without explicitly applying the kernel transformation $\vphi$.
In terms of notation, when we have $\Big\langle \vtheta, \phiproj(\zz_n) \Big\rangle$, we write 
\begin{equation}
\sum_{m=1}^N \beta_{m} \left(\kernel(\xxm, \zz_n) - \frac{\valpha^{\top} \Kmn  \valpha}{\valpha^{\top} K \valpha} \right)
\end{equation}
where we define $\boldsymbol{K}^{(m, n)}_{ij}(\zz) \defeq \kernel(\xxi, \zz_n) \kernel(\xxm, \xxj)$.\ryan{should $\zz_n$ be $\zz$?}
This proves the result.
\end{proof}

\subsection{Experimental setting}
\label{app:exp-setting}

\paragraph{Data.} We conduct experiments on the uncased version of the GloVe representations, which are 300-dimensional, licensed under Apache License, Version 2.0. Following \citet{ravfogel-etal-2020-null}, to approximate the gender labels for the vocabulary, we project all representations on the $\overrightarrow{\text{he}}-\overrightarrow{\text{she}}$ direction, and take the $7,500$ most male-biased and female-biased words. Note that unlike \citep{bolukbasi2016man}, we use the $\overrightarrow{\text{he}}-\overrightarrow{\text{she}}$\ direction only to induce approximate gender labels, but then proceed to measure the bias in various ways that go beyond neutralizing just the $\overrightarrow{\text{he}}-\overrightarrow{\text{she}}$ direction.
We use the same train--dev--test split of \citet{ravfogel-etal-2020-null}, but discard the gender-neutral words (i.e., we cast the problem as a binary classification).  We obtain training, evaluation, and test sets of sizes 7,350, 3,150 and 4,500, respectively. We perform four independent runs of the entire method for all kernel types, with different random seeds.

\paragraph{The kernelized maximin game.} For each kernel, we experiment with the following combinations of hyperparameter values:\looseness=-1
\begin{itemize}
    \item \poly\:: $\degree\ \in \{2,3\}$; $\gamma \in \{0.05, 0.1, 0.15\}$; $\alpha \in \{0.8,1,1.2\}$.
    \item \rbf\:: $\gamma \in \{0.1, 0.15, 0.2\}$
    \item \laplace\:: $\gamma \in \{0.1, 0.15, 0.2\}$
    \item \sigmoid\:: $\alpha \in \{0, 0.01\}$; $\gamma \in \{0.005, 0.003\}$.
\end{itemize}
We approximate the kernel space using $L=1024$ \nystrom\ landmarks. We run the adversarial game \cref{eq:relaxed-game-nystrom} for each of the kernel mappings we consider, by performing alternate minimization and maximization over $\vtheta$ and $\PP$, respectively.
As our optimization procedure, we use stochastic gradient descent with a learning rate of 0.08 and minibatches of size 256. We run for 35,000 batches, and choose the projection matrix $\PP$ which leads to the biggest \emph{decrease} in the linear classification accuracy on the evaluation set. In all cases, we identify a matrix which decreases classification accuracy to near-random accuracy. All training is done on a single NVIDIA GeForce GTX 1080 Ti GPU.

\paragraph{Pre-image mapping.} We train an MLP with 2 hidden layers of sizes 512 and 300 to map the original inputs $\xx_n$ to inputs which, after being mapped to kernel space, are close to the neutralized features. We use dropout of 0.1, ReLU activation and layer normalization after each hidden layer. We use a skip connection between the input and the output layer, i.e., we set the final output of the MLP to be the sum of its inputs and outputs.\footnote{This is useful because we assume the gender information is not very salient in the original representations, since these encode much more information. With the skip connection, the MLP only needs to learn what to remove, rather than learning to generate plausible representations.} We train for 15,000 batches of size 128 and choose the model that yields the lowest loss on the evaluation set.  

\paragraph{Non-linear gender prediction.} We consider the following non-linear predictors: SVMs with different kernels, as well as an MLP with 128 hidden units and ReLU activations. We use the sklearn implementation of predictors. They are trained on the reconstructed pre-image of the training set, and tested on the reconstructed pre-image of the test set. Note that while in training we used an approximation of the kernel function, we predict gender from the pre-images using SVM classifiers that rely on the actual, exact kernel.\looseness=-1 

\subsubsection{Pipeline Evaluation}
\label{app:pipeline-eval}
In this appendix, we include sanity check experiments that aim to assess whether the maximin game \cref{eq:relaxed-game-nystrom} effectively removes \emph{linearly}-present concepts from the \emph{non-linear} kernel features, and whether the training of the pre-image network succeeds. 

\paragraph{Concept erasure in kernel space.}
Do we effectively neutralize the concept in the approximate kernel space? For each kernel, we solve the game \cref{eq:relaxed-game-nystrom} and use the final projection matrix $\PP$ to create neutralized features.
We neutralize the features in RKHS by mapping $\phxnyst \mapsto \PP \phxnyst $, and train a \emph{linear} classifier to recover the gender labels from the neutralized representations. We get a classification accuracy of 50.59$\pm$0.04, very close to majority accuracy of 50.58. This suggests that the process is effective in protecting against the kernel which was applied in training (training a linear classifier on the kernel-transformed representations is equivalent to training a kernel classifier on the original representations). 
Notice, however, that we cannot test \emph{other} non-linear kernel classifiers on these representations in a similar way: If the approximate kernel mapping $\widetilde{\vphi}(\cdot)$ corresponds, for example, to a polynomial kernel, we cannot measure the success of an RBF kernel in recovering the bias information after the intervention without performing the pre-image mapping.

\paragraph{Pre-image mapping.} Our neutralization algorithm relies on calculating the pre-image of the kernel features after the intervention. To evaluate the quality of the pre-image mapping, we measure the relative reconstruction error $\left|\left|\frac{{\PP \phxnyst - \phnyst(f(\xx_n))}}{\PP \phxnyst}\right|\right|^2_2$ over all points $\xx_n$ on the evaluation set. When averaged over all 4 seeds and the different kernels we experimented with, we get a reconstruction error of 1.81$\pm$1.61\% (range 0.45-7.94).

\clearpage

\subsection{Gender Prediction from the Pre-image}
\label{app:gender-results}
In \Cref{tab:pre-image-pred-full} we report the full evaluation results on the pre-image  neutralized representations, where we have used the same kernel (and the same hyperparameters) for neutralization and gender recovery from the pre-image.

\begin{table}[h]
\centering
\begin{tabular}{lllllll}
\toprule
  Kernel &  $\gamma$ & $\alpha$ &   \degree &                \texttt{WEAT's} d &                  \texttt{WEAT's} p-value &         Gender Acc. \\
\midrule
     \poly &  0.05 &   0.8 &    2 & 0.73 $\pm$ 0.01 & 0.084 $\pm$ 0.002 & 0.49 $\pm$ 0.00 \\
     \poly &  0.05 &     1 &    2 & 0.74 $\pm$ 0.01 & 0.080 $\pm$ 0.002 & 0.49 $\pm$ 0.00 \\
     \poly &  0.05 &   1.2 &    2 & 0.74 $\pm$ 0.01 & 0.081 $\pm$ 0.002 & 0.49 $\pm$ 0.00 \\
     \poly &  0.05 &   0.8 &    3 & 0.75 $\pm$ 0.01 & 0.077 $\pm$ 0.003 & 0.49 $\pm$ 0.00 \\
     \poly &  0.05 &     1 &    3 & 0.74 $\pm$ 0.01 & 0.080 $\pm$ 0.003 & 0.49 $\pm$ 0.00 \\
     \poly &  0.05 &   1.2 &    3 & 0.74 $\pm$ 0.01 & 0.081 $\pm$ 0.003 & 0.49 $\pm$ 0.00 \\
     \poly &   0.1 &   0.8 &    2 & 0.74 $\pm$ 0.01 & 0.080 $\pm$ 0.003 & 0.49 $\pm$ 0.00 \\
     \poly &   0.1 &     1 &    2 & 0.74 $\pm$ 0.00 & 0.080 $\pm$ 0.001 & 0.49 $\pm$ 0.00 \\
     \poly &   0.1 &   1.2 &    2 & 0.74 $\pm$ 0.01 & 0.081 $\pm$ 0.002 & 0.49 $\pm$ 0.00 \\
     \poly &   0.1 &   0.8 &    3 & 0.74 $\pm$ 0.00 & 0.081 $\pm$ 0.002 & 0.53 $\pm$ 0.01 \\
     \poly &   0.1 &     1 &    3 & 0.73 $\pm$ 0.01 & 0.082 $\pm$ 0.004 & 0.65 $\pm$ 0.03 \\
     \poly &   0.1 &   1.2 &    3 & 0.73 $\pm$ 0.01 & 0.084 $\pm$ 0.004 & 0.72 $\pm$ 0.01 \\
     \poly &  0.15 &   0.8 &    2 & 0.73 $\pm$ 0.01 & 0.082 $\pm$ 0.003 & 0.56 $\pm$ 0.03 \\
     \poly &  0.15 &     1 &    2 & 0.74 $\pm$ 0.02 & 0.081 $\pm$ 0.006 & 0.54 $\pm$ 0.02 \\
     \poly &  0.15 &   1.2 &    2 & 0.73 $\pm$ 0.01 & 0.084 $\pm$ 0.004 & 0.55 $\pm$ 0.02 \\
     \poly &  0.15 &   0.8 &    3 & 0.73 $\pm$ 0.01 & 0.082 $\pm$ 0.003 & 0.88 $\pm$ 0.02 \\
     \poly &  0.15 &     1 &    3 & 0.73 $\pm$ 0.01 & 0.082 $\pm$ 0.003 & 0.92 $\pm$ 0.00 \\
     \poly &  0.15 &   1.2 &    3 & 0.74 $\pm$ 0.01 & 0.079 $\pm$ 0.003 & 0.93 $\pm$ 0.00 \\
      \rbf &   0.1 &  - & - & 0.75 $\pm$ 0.01 & 0.078 $\pm$ 0.003 & 0.49 $\pm$ 0.01 \\
      \rbf &  0.15 &  - & - & 0.74 $\pm$ 0.01 & 0.079 $\pm$ 0.003 & 0.68 $\pm$ 0.03 \\
      \rbf &   0.2 &  - & - & 0.74 $\pm$ 0.01 & 0.081 $\pm$ 0.003 & 0.89 $\pm$ 0.01 \\
  \laplace &   0.1 &  - & - & 0.72 $\pm$ 0.03 & 0.086 $\pm$ 0.008 & 0.62 $\pm$ 0.04 \\
  \laplace &  0.15 &  - & - & 0.74 $\pm$ 0.05 & 0.080 $\pm$ 0.015 & 0.77 $\pm$ 0.05 \\
  \laplace &   0.2 &  - & - & 0.67 $\pm$ 0.05 & 0.107 $\pm$ 0.020 & 0.88 $\pm$ 0.04 \\
   \linear &  - &  - & - & 0.74 $\pm$ 0.01 & 0.079 $\pm$ 0.004 & 0.54 $\pm$ 0.02 \\
  \sigmoid & 0.005 &     0 & - & 0.78 $\pm$ 0.05 & 0.069 $\pm$ 0.014 & 0.49 $\pm$ 0.00 \\
  \sigmoid & 0.005 &  0.01 & - & 0.73 $\pm$ 0.03 & 0.082 $\pm$ 0.010 & 0.49 $\pm$ 0.00 \\
  \sigmoid & 0.003 &     0 & - & 0.73 $\pm$ 0.05 & 0.083 $\pm$ 0.017 & 0.49 $\pm$ 0.00 \\
  \sigmoid & 0.003 &  0.01 & - & 0.76 $\pm$ 0.03 & 0.074 $\pm$ 0.008 & 0.49 $\pm$ 0.00 \\
  \easymkl &  - &  - & - & 0.73 $\pm$ 0.01 & 0.084 $\pm$ 0.005 & 0.69 $\pm$ 0.01 \\
\uniformmk &  - &  - & - & 0.73 $\pm$ 0.01 & 0.084 $\pm$ 0.002 & 0.49 $\pm$ 0.00 \\
    \hline
 Original & - & - & - & 1.56 & 0.000 & $\geq 0.99$ \\
\bottomrule
\end{tabular}
    
\caption{Evaluation of the neutralized pre-image representations. We show the WEAT test's statistics and $p$-value, as well as the gender prediction accuracy of a kernel classifier of the same type as the one applied during neutralization.}
\label{tab:pre-image-pred-full}
\end{table}

\clearpage

\subsection{Closest Neighbors}
\label{app:neighbors}
In \Cref{tab:neighbors}, we show the closest neighbors to randomly-sampled word representations before and after gender erasure under the polynomial kernel. 
The results for other kernels are qualitatively similar. 

\begin{table}[h]
\centering
\scalebox{0.75}{
\begin{tabular}{lll}
\toprule
          Word &                   Neighbors before &                    Neighbors after \\
\midrule
     spiritual &          faith, religious, healing &      emotional, religious, healing \\
        lesson &              learn, teach, lessons &           teaching, teach, lessons \\
         faces &                faced, facing, face &                faced, facing, face \\
        forget &                know, let, remember &                know, let, remember \\
     converter &          ipod, conversion, convert &          ipod, conversion, convert \\
         clean &               keep, wash, cleaning &               keep, wash, cleaning \\
        formal &        elegant, dress, appropriate &        elegant, appropriate, dress \\
      identity &  identify, context, identification &  context, identify, identification \\
         other &                 these, those, many &                 these, those, many \\
      licensed &     registered, certified, license &     registered, certified, license \\
       ratings &             reviews, rated, rating &             reviews, rated, rating \\
      properly &     proper, effectively, correctly &     effectively, proper, correctly \\
         build &            create, built, building &            built, create, building \\
     solutions &    systems, technologies, solution &   services, technologies, solution \\
   afghanistan &             troops, pakistan, iraq &             troops, pakistan, iraq \\
     wallpaper &         desktop, pictures, picture &         desktop, pictures, picture \\
         sound &               audio, noise, sounds &               audio, noise, sounds \\
        gender &                  sexual, male, age &             male, differences, age \\
          boat &              cruise, ship, fishing &              cruise, ship, fishing \\
      downtown &       portland, city, neighborhood &       portland, neighborhood, city \\
       lawyers &        attorney, lawyer, attorneys &        attorney, lawyer, attorneys \\
         smart &             how, easy, intelligent &            wise, easy, intelligent \\
      spending &               budget, spent, spend &               budget, spent, spend \\
       contest &       winners, winner, competition &       winners, winner, competition \\
          want &                    n't, know, need &                    n't, know, need \\
        advice &        guidance, suggestions, tips &        guidance, suggestions, tips \\
 professionals &    managers, professional, experts &    managers, professional, experts \\
             g &                            d, b, f &                            d, b, f \\
    australian &        zealand, british, australia &        zealand, british, australia \\
            na &                          mo, o, da &                          mo, o, da \\
\bottomrule
\end{tabular}

}
\caption{Closest neighbors to randomly-sampled words from GloVe vocabulary, for the original representations, and for the pre-images after our intervention. }
\label{tab:neighbors}
\end{table}

\clearpage

\subsection{WEAT Results}
\label{app:weat}
Here we report the results of the WEAT test for the career and family-related words (\Cref{tab:weat-family-profession}) and art and mathematics-related words (\Cref{tab:weat-math-art}).

\begin{table}[h]
\begin{center}
\begin{tabular}{llllll}
\toprule
    Kernel & $\gamma$ & $\alpha$ &  \degree &                \textsf{WEAT-$d$} &                  $p$-value \\
\midrule
      \poly &   0.05 &   0.8 &     2 &  0.72 $\pm$ 0.00 &  0.093 $\pm$ 0.002 \\
      \poly &   0.05 &     1 &     2 &  0.72 $\pm$ 0.01 &  0.091 $\pm$ 0.003 \\
      \poly &   0.05 &   1.2 &     2 &  0.72 $\pm$ 0.01 &  0.093 $\pm$ 0.004 \\
      \poly &   0.05 &   0.8 &     3 &  0.73 $\pm$ 0.00 &  0.089 $\pm$ 0.001 \\
      \poly &   0.05 &     1 &     3 &  0.74 $\pm$ 0.01 &  0.086 $\pm$ 0.004 \\
      \poly &   0.05 &   1.2 &     3 &  0.73 $\pm$ 0.01 &  0.089 $\pm$ 0.002 \\
      \poly &    0.1 &   0.8 &     2 &  0.72 $\pm$ 0.01 &  0.091 $\pm$ 0.005 \\
      \poly &    0.1 &     1 &     2 &  0.73 $\pm$ 0.00 &  0.089 $\pm$ 0.001 \\
      \poly &    0.1 &   1.2 &     2 &  0.72 $\pm$ 0.01 &  0.090 $\pm$ 0.003 \\
      \poly &    0.1 &   0.8 &     3 &  0.72 $\pm$ 0.01 &  0.090 $\pm$ 0.003 \\
      \poly &    0.1 &     1 &     3 &  0.72 $\pm$ 0.01 &  0.090 $\pm$ 0.003 \\
      \poly &    0.1 &   1.2 &     3 &  0.72 $\pm$ 0.01 &  0.091 $\pm$ 0.002 \\
      \poly &   0.15 &   0.8 &     2 &  0.71 $\pm$ 0.01 &  0.095 $\pm$ 0.004 \\
      \poly &   0.15 &     1 &     2 &  0.74 $\pm$ 0.00 &  0.087 $\pm$ 0.001 \\
      \poly &   0.15 &   1.2 &     2 &  0.72 $\pm$ 0.01 &  0.091 $\pm$ 0.002 \\
      \poly &   0.15 &   0.8 &     3 &  0.72 $\pm$ 0.00 &  0.093 $\pm$ 0.001 \\
      \poly &   0.15 &     1 &     3 &  0.72 $\pm$ 0.01 &  0.092 $\pm$ 0.002 \\
      \poly &   0.15 &   1.2 &     3 &  0.73 $\pm$ 0.01 &  0.090 $\pm$ 0.005 \\
       \rbf &    0.1 &  - &  - &  0.72 $\pm$ 0.01 &  0.090 $\pm$ 0.003 \\
       \rbf &   0.15 &  - &  - &  0.73 $\pm$ 0.01 &  0.090 $\pm$ 0.005 \\
       \rbf &    0.2 &  - &  - &  0.72 $\pm$ 0.01 &  0.091 $\pm$ 0.003 \\
   \laplace &    0.1 &  - &  - &  0.75 $\pm$ 0.05 &  0.083 $\pm$ 0.017 \\
   \laplace &   0.15 &  - &  - &  0.77 $\pm$ 0.02 &  0.076 $\pm$ 0.007 \\
   \laplace &    0.2 &  - &  - &  0.70 $\pm$ 0.02 &  0.098 $\pm$ 0.008 \\
    \linear &   - &  - &  - &  0.72 $\pm$ 0.02 &  0.090 $\pm$ 0.006 \\
   \sigmoid &  0.005 &     0 &  - &  0.73 $\pm$ 0.03 &  0.087 $\pm$ 0.010 \\
   \sigmoid &  0.005 &  0.01 &  - &  0.73 $\pm$ 0.02 &  0.087 $\pm$ 0.006 \\
   \sigmoid &  0.003 &     0 &  - &  0.76 $\pm$ 0.06 &  0.079 $\pm$ 0.019 \\
   \sigmoid &  0.003 &  0.01 &  - &  0.76 $\pm$ 0.07 &  0.080 $\pm$ 0.020 \\
   \easymkl &   - &  - &  - &  0.72 $\pm$ 0.02 &  0.091 $\pm$ 0.005 \\
 \uniformmk &   - &  - &  - &  0.72 $\pm$ 0.00 &  0.092 $\pm$ 0.002 \\
  \hline
 Original & - & - & - & 1.69 & 0.000 \\
\bottomrule
\end{tabular}

\end{center}
\caption{Word association bias test (WEAT) for career and family-related terms}
\label{tab:weat-family-profession}
\end{table}

\clearpage

\begin{table}[h]
\begin{center}
\begin{tabular}{llllll}
\toprule
    Kernel &  $\gamma$ & $\alpha$ &   \degree &                \textsf{WEAT-d} &                  $p$-value \\
\midrule
      \poly &   0.05 &   0.8 &     2 &  0.78 $\pm$ 0.00 &  0.068 $\pm$ 0.001 \\
      \poly &   0.05 &     1 &     2 &  0.78 $\pm$ 0.01 &  0.067 $\pm$ 0.002 \\
      \poly &   0.05 &   1.2 &     2 &  0.78 $\pm$ 0.00 &  0.067 $\pm$ 0.001 \\
      \poly &   0.05 &   0.8 &     3 &  0.78 $\pm$ 0.01 &  0.066 $\pm$ 0.002 \\
      \poly &   0.05 &     1 &     3 &  0.78 $\pm$ 0.01 &  0.066 $\pm$ 0.002 \\
      \poly &   0.05 &   1.2 &     3 &  0.78 $\pm$ 0.00 &  0.066 $\pm$ 0.001 \\
      \poly &    0.1 &   0.8 &     2 &  0.78 $\pm$ 0.00 &  0.068 $\pm$ 0.001 \\
      \poly &    0.1 &     1 &     2 &  0.78 $\pm$ 0.01 &  0.066 $\pm$ 0.002 \\
      \poly &    0.1 &   1.2 &     2 &  0.77 $\pm$ 0.01 &  0.069 $\pm$ 0.004 \\
      \poly &    0.1 &   0.8 &     3 &  0.78 $\pm$ 0.00 &  0.067 $\pm$ 0.001 \\
      \poly &    0.1 &     1 &     3 &  0.77 $\pm$ 0.01 &  0.070 $\pm$ 0.001 \\
      \poly &    0.1 &   1.2 &     3 &  0.78 $\pm$ 0.01 &  0.068 $\pm$ 0.002 \\
      \poly &   0.15 &   0.8 &     2 &  0.78 $\pm$ 0.01 &  0.068 $\pm$ 0.003 \\
      \poly &   0.15 &     1 &     2 &  0.78 $\pm$ 0.01 &  0.068 $\pm$ 0.003 \\
      \poly &   0.15 &   1.2 &     2 &  0.78 $\pm$ 0.00 &  0.068 $\pm$ 0.001 \\
      \poly &   0.15 &   0.8 &     3 &  0.78 $\pm$ 0.01 &  0.067 $\pm$ 0.002 \\
      \poly &   0.15 &     1 &     3 &  0.78 $\pm$ 0.01 &  0.067 $\pm$ 0.002 \\
      \poly &   0.15 &   1.2 &     3 &  0.77 $\pm$ 0.01 &  0.069 $\pm$ 0.002 \\
       \rbf &    0.1 &  - &  - &  0.79 $\pm$ 0.01 &  0.066 $\pm$ 0.003 \\
       \rbf &   0.15 &  - &  - &  0.78 $\pm$ 0.01 &  0.067 $\pm$ 0.002 \\
       \rbf &    0.2 &  - &  - &  0.78 $\pm$ 0.01 &  0.067 $\pm$ 0.002 \\
   \laplace &    0.1 &  - &  - &  0.80 $\pm$ 0.04 &  0.064 $\pm$ 0.012 \\
   \laplace &   0.15 &  - &  - &  0.81 $\pm$ 0.03 &  0.061 $\pm$ 0.009 \\
   \laplace &    0.2 &  - &  - &  0.77 $\pm$ 0.04 &  0.070 $\pm$ 0.013 \\
    linear &   - &  - &  - &  0.79 $\pm$ 0.01 &  0.066 $\pm$ 0.002 \\
   \sigmoid &  0.005 &     0 &  - &  0.82 $\pm$ 0.04 &  0.057 $\pm$ 0.010 \\
   \sigmoid &  0.005 &  0.01 &  - &  0.77 $\pm$ 0.05 &  0.070 $\pm$ 0.012 \\
   \sigmoid &  0.003 &     0 &  - &  0.76 $\pm$ 0.03 &  0.073 $\pm$ 0.009 \\
   \sigmoid &  0.003 &  0.01 &  - &  0.79 $\pm$ 0.03 &  0.066 $\pm$ 0.009 \\
   \easymkl &   - &  - &  - &  0.78 $\pm$ 0.00 &  0.066 $\pm$ 0.001 \\
 \uniformmk &   - &  - &  - &  0.78 $\pm$ 0.01 &  0.069 $\pm$ 0.002 \\
   \hline
 Original & - & - & - & 1.56 & 0.000 \\
\bottomrule
\end{tabular}

\end{center}
\caption{Word association bias test (WEAT) for art and mathematics-related terms}
\label{tab:weat-math-art}
\end{table}

\clearpage

\subsection{Transfer Results}
\label{app:trasnfer}

\begin{table}[h]
\scalebox{0.65}{

\begin{tabular}{lllllllll}
\toprule
{} &        UniformMK &          EasyMKL &              \rbf &             \poly &          \laplace &          \sigmoid &           \linear &              \mlp \\
\midrule
poly, $\gamma  = $0.05, \degree = 2, $\alpha =  $0.8       &  0.49 $\pm$ 0.00 &  0.98 $\pm$ 0.00 &  0.96 $\pm$ 0.00 &  0.98 $\pm$ 0.00 &  0.94 $\pm$ 0.01 &  0.49 $\pm$ 0.00 &  0.61 $\pm$ 0.06 &  0.97 $\pm$ 0.00 \\
poly, $\gamma  = $0.05, \degree = 2, $\alpha =  $1         &  0.49 $\pm$ 0.00 &  0.98 $\pm$ 0.00 &  0.96 $\pm$ 0.01 &  0.98 $\pm$ 0.00 &  0.94 $\pm$ 0.00 &  0.49 $\pm$ 0.00 &  0.52 $\pm$ 0.01 &  0.97 $\pm$ 0.00 \\
poly, $\gamma  = $0.05, \degree = 2, $\alpha =  $1.2       &  0.49 $\pm$ 0.00 &  0.98 $\pm$ 0.00 &  0.96 $\pm$ 0.00 &  0.98 $\pm$ 0.00 &  0.94 $\pm$ 0.00 &  0.49 $\pm$ 0.00 &  0.53 $\pm$ 0.02 &  0.97 $\pm$ 0.00 \\
poly, $\gamma  = $0.05, \degree = 3, $\alpha =  $0.8       &  0.49 $\pm$ 0.00 &  0.98 $\pm$ 0.00 &  0.96 $\pm$ 0.01 &  0.98 $\pm$ 0.00 &  0.93 $\pm$ 0.01 &  0.49 $\pm$ 0.00 &  0.54 $\pm$ 0.03 &  0.97 $\pm$ 0.00 \\
poly, $\gamma  = $0.05, \degree = 3, $\alpha =  $1         &  0.49 $\pm$ 0.00 &  0.98 $\pm$ 0.00 &  0.96 $\pm$ 0.00 &  0.98 $\pm$ 0.00 &  0.94 $\pm$ 0.01 &  0.49 $\pm$ 0.00 &  0.59 $\pm$ 0.05 &  0.97 $\pm$ 0.00 \\
poly, $\gamma  = $0.05, \degree = 3, $\alpha =  $1.2       &  0.49 $\pm$ 0.00 &  0.98 $\pm$ 0.00 &  0.96 $\pm$ 0.00 &  0.98 $\pm$ 0.00 &  0.94 $\pm$ 0.00 &  0.49 $\pm$ 0.00 &  0.53 $\pm$ 0.01 &  0.97 $\pm$ 0.00 \\
poly, $\gamma  = $0.1, \degree = 2, $\alpha =  $0.8        &  0.49 $\pm$ 0.00 &  0.98 $\pm$ 0.00 &  0.96 $\pm$ 0.00 &  0.98 $\pm$ 0.00 &  0.93 $\pm$ 0.01 &  0.49 $\pm$ 0.00 &  0.54 $\pm$ 0.04 &  0.97 $\pm$ 0.00 \\
poly, $\gamma  = $0.1, \degree = 2, $\alpha =  $1          &  0.49 $\pm$ 0.00 &  0.98 $\pm$ 0.00 &  0.96 $\pm$ 0.00 &  0.98 $\pm$ 0.00 &  0.93 $\pm$ 0.01 &  0.49 $\pm$ 0.00 &  0.53 $\pm$ 0.02 &  0.97 $\pm$ 0.00 \\
poly, $\gamma  = $0.1, \degree = 2, $\alpha =  $1.2        &  0.49 $\pm$ 0.00 &  0.98 $\pm$ 0.00 &  0.96 $\pm$ 0.00 &  0.98 $\pm$ 0.00 &  0.94 $\pm$ 0.00 &  0.49 $\pm$ 0.00 &  0.54 $\pm$ 0.01 &  0.97 $\pm$ 0.00 \\
poly, $\gamma  = $0.1, \degree = 3, $\alpha =  $0.8        &  0.49 $\pm$ 0.00 &  0.98 $\pm$ 0.00 &  0.96 $\pm$ 0.00 &  0.98 $\pm$ 0.00 &  0.94 $\pm$ 0.00 &  0.49 $\pm$ 0.00 &  0.54 $\pm$ 0.00 &  0.97 $\pm$ 0.00 \\
poly, $\gamma  = $0.1, \degree = 3, $\alpha =  $1          &  0.49 $\pm$ 0.00 &  0.98 $\pm$ 0.00 &  0.96 $\pm$ 0.00 &  0.98 $\pm$ 0.00 &  0.93 $\pm$ 0.00 &  0.49 $\pm$ 0.00 &  0.56 $\pm$ 0.02 &  0.97 $\pm$ 0.00 \\
poly, $\gamma  = $0.1, \degree = 3, $\alpha =  $1.2        &  0.49 $\pm$ 0.00 &  0.98 $\pm$ 0.00 &  0.96 $\pm$ 0.00 &  0.98 $\pm$ 0.00 &  0.93 $\pm$ 0.00 &  0.49 $\pm$ 0.00 &  0.54 $\pm$ 0.03 &  0.97 $\pm$ 0.00 \\
poly, $\gamma  = $0.15, \degree = 2, $\alpha =  $0.8       &  0.49 $\pm$ 0.00 &  0.98 $\pm$ 0.00 &  0.96 $\pm$ 0.00 &  0.98 $\pm$ 0.00 &  0.93 $\pm$ 0.00 &  0.49 $\pm$ 0.00 &  0.57 $\pm$ 0.02 &  0.97 $\pm$ 0.00 \\
poly, $\gamma  = $0.15, \degree = 2, $\alpha =  $1         &  0.49 $\pm$ 0.00 &  0.98 $\pm$ 0.00 &  0.96 $\pm$ 0.00 &  0.98 $\pm$ 0.00 &  0.94 $\pm$ 0.00 &  0.49 $\pm$ 0.00 &  0.54 $\pm$ 0.01 &  0.97 $\pm$ 0.00 \\
poly, $\gamma  = $0.15, \degree = 2, $\alpha =  $1.2       &  0.49 $\pm$ 0.00 &  0.98 $\pm$ 0.00 &  0.96 $\pm$ 0.00 &  0.98 $\pm$ 0.00 &  0.93 $\pm$ 0.00 &  0.49 $\pm$ 0.00 &  0.54 $\pm$ 0.01 &  0.97 $\pm$ 0.00 \\
poly, $\gamma  = $0.15, \degree = 3, $\alpha =  $0.8       &  0.49 $\pm$ 0.00 &  0.98 $\pm$ 0.00 &  0.96 $\pm$ 0.00 &  0.98 $\pm$ 0.00 &  0.93 $\pm$ 0.01 &  0.49 $\pm$ 0.00 &  0.58 $\pm$ 0.03 &  0.97 $\pm$ 0.00 \\
poly, $\gamma  = $0.15, \degree = 3, $\alpha =  $1         &  0.49 $\pm$ 0.00 &  0.98 $\pm$ 0.00 &  0.96 $\pm$ 0.00 &  0.98 $\pm$ 0.00 &  0.93 $\pm$ 0.00 &  0.49 $\pm$ 0.00 &  0.57 $\pm$ 0.01 &  0.97 $\pm$ 0.00 \\
poly, $\gamma  = $0.15, \degree = 3, $\alpha =  $1.2       &  0.49 $\pm$ 0.00 &  0.98 $\pm$ 0.00 &  0.96 $\pm$ 0.00 &  0.98 $\pm$ 0.00 &  0.93 $\pm$ 0.00 &  0.49 $\pm$ 0.00 &  0.56 $\pm$ 0.01 &  0.97 $\pm$ 0.00 \\
rbf, $\gamma  = $0.1         &  0.49 $\pm$ 0.00 &  0.98 $\pm$ 0.00 &  0.96 $\pm$ 0.00 &  0.98 $\pm$ 0.00 &  0.94 $\pm$ 0.00 &  0.49 $\pm$ 0.00 &  0.58 $\pm$ 0.04 &  0.97 $\pm$ 0.00 \\
rbf, $\gamma  = $0.15        &  0.49 $\pm$ 0.00 &  0.98 $\pm$ 0.00 &  0.96 $\pm$ 0.00 &  0.98 $\pm$ 0.00 &  0.94 $\pm$ 0.00 &  0.49 $\pm$ 0.00 &  0.60 $\pm$ 0.03 &  0.97 $\pm$ 0.00 \\
rbf, $\gamma  = $0.2         &  0.49 $\pm$ 0.00 &  0.98 $\pm$ 0.00 &  0.96 $\pm$ 0.00 &  0.98 $\pm$ 0.00 &  0.93 $\pm$ 0.01 &  0.49 $\pm$ 0.00 &  0.60 $\pm$ 0.04 &  0.97 $\pm$ 0.00 \\
laplace, $\gamma  = $0.1     &  0.60 $\pm$ 0.06 &  0.98 $\pm$ 0.00 &  0.96 $\pm$ 0.00 &  0.98 $\pm$ 0.00 &  0.93 $\pm$ 0.01 &  0.49 $\pm$ 0.00 &  0.61 $\pm$ 0.03 &  0.97 $\pm$ 0.00 \\
laplace, $\gamma  = $0.15    &  0.56 $\pm$ 0.04 &  0.98 $\pm$ 0.00 &  0.96 $\pm$ 0.00 &  0.98 $\pm$ 0.00 &  0.94 $\pm$ 0.00 &  0.49 $\pm$ 0.00 &  0.61 $\pm$ 0.01 &  0.97 $\pm$ 0.00 \\
laplace, $\gamma  = $0.2     &  0.59 $\pm$ 0.06 &  0.98 $\pm$ 0.00 &  0.96 $\pm$ 0.00 &  0.98 $\pm$ 0.00 &  0.94 $\pm$ 0.00 &  0.49 $\pm$ 0.00 &  0.62 $\pm$ 0.02 &  0.97 $\pm$ 0.00 \\
linear       &  0.49 $\pm$ 0.00 &  0.98 $\pm$ 0.00 &  0.96 $\pm$ 0.00 &  0.98 $\pm$ 0.00 &  0.93 $\pm$ 0.00 &  0.49 $\pm$ 0.00 &  0.54 $\pm$ 0.02 &  0.97 $\pm$ 0.00 \\
sigmoid, $\gamma  = $0.005, $\alpha =  $0    &  0.64 $\pm$ 0.05 &  0.97 $\pm$ 0.00 &  0.93 $\pm$ 0.01 &  0.98 $\pm$ 0.00 &  0.90 $\pm$ 0.01 &  0.49 $\pm$ 0.00 &  0.65 $\pm$ 0.05 &  0.97 $\pm$ 0.00 \\
sigmoid, $\gamma  = $0.005, $\alpha =  $0.01 &  0.63 $\pm$ 0.02 &  0.97 $\pm$ 0.00 &  0.94 $\pm$ 0.00 &  0.98 $\pm$ 0.00 &  0.90 $\pm$ 0.01 &  0.49 $\pm$ 0.00 &  0.65 $\pm$ 0.04 &  0.97 $\pm$ 0.00 \\
sigmoid, $\gamma  = $0.003, $\alpha =  $0    &  0.63 $\pm$ 0.07 &  0.97 $\pm$ 0.00 &  0.92 $\pm$ 0.01 &  0.98 $\pm$ 0.00 &  0.89 $\pm$ 0.01 &  0.49 $\pm$ 0.00 &  0.66 $\pm$ 0.07 &  0.97 $\pm$ 0.00 \\
sigmoid, $\gamma  = $0.003, $\alpha =  $0.01 &  0.65 $\pm$ 0.03 &  0.97 $\pm$ 0.00 &  0.92 $\pm$ 0.01 &  0.98 $\pm$ 0.00 &  0.89 $\pm$ 0.01 &  0.49 $\pm$ 0.00 &  0.65 $\pm$ 0.05 &  0.97 $\pm$ 0.00 \\
EasyMKL      &  0.49 $\pm$ 0.00 &  0.98 $\pm$ 0.00 &  0.96 $\pm$ 0.00 &  0.98 $\pm$ 0.00 &  0.94 $\pm$ 0.00 &  0.49 $\pm$ 0.00 &  0.57 $\pm$ 0.03 &  0.97 $\pm$ 0.00 \\
UniformMK    &  0.49 $\pm$ 0.00 &  0.98 $\pm$ 0.00 &  0.96 $\pm$ 0.00 &  0.98 $\pm$ 0.00 &  0.94 $\pm$ 0.01 &  0.49 $\pm$ 0.00 &  0.58 $\pm$ 0.08 &  0.97 $\pm$ 0.00 \\
\bottomrule
\end{tabular}
}
\caption{Gender prediction from the neutralized pre-image representations using non-linear adversaries that differ from the neutralizing kernel.}
\label{tab:pre-image-pred-diff}
\end{table}

In this appendix, we provide gender prediction accuracy, on the neutralized pre-image representations, with predictors that are \emph{different} from those used in training (Experiment \cref{sec:transfer}).

\paragraph{Setup.} After projecting out the gender concept in kernel space, and computing the pre-image of the neutralized representations, we apply different non-linear kernels as well as an MLP to predict gender. We use the following parameters:

\begin{itemize}
    \item \rbf\:: $\gamma=0.3$.
    \item \poly\:: $\degree\ =3$, $\gamma=0.5$, $\alpha=0.3$.
    \item \laplace\:: $\gamma=0.3$.
    \item \sigmoid\:: $\alpha=0$, $\gamma=0.01$.
    \item \mlp\:: A network with a single 128-dimensional hidden layer with ReLU activations.
\end{itemize}

All classifiers were trained using sklearn.

\paragraph{Results.} The results are shown in \Cref{tab:pre-image-pred-diff}. Rows denote the kernel that was applied for neutralization in \cref{eq:relaxed-game-nystrom}, while columns denote the type of adversarial classifier applied on the final pre-image representations. Numbers denote accuracy in gender prediction.

\end{document}